\documentclass{article} 
\usepackage{iclr2026_conference,times}

\usepackage[utf8]{inputenc} 
\usepackage[T1]{fontenc}    
\usepackage{hyperref}       
\usepackage{nccmath}
\usepackage{pgfplots}
\pgfplotsset{compat=1.18}
\hypersetup{
     colorlinks=true,
     linkcolor=blue,
     filecolor=blue,
     citecolor=blue,
     urlcolor=blue,
     }
\usepackage{url}
\usepackage{upgreek}
\usepackage{wrapfig}
\usepackage{booktabs}       
\usepackage{amsfonts}       
\usepackage{nicefrac}
\usepackage{microtype}      
\usepackage{xcolor}         
\usepackage{multicol}
\usepackage{graphicx}
\usepackage{amsthm}
\usepackage{bbm}
\usepackage{comment}
\usepackage{enumitem}
\usepackage{bm}
\usepackage{booktabs}
\usepackage{subfigure}
\usepackage{mathtools}
\usepackage[export]{adjustbox}
\usepackage{stmaryrd}
\usepackage{booktabs} 
\usepackage{xifthen}
\usepackage{xargs}
\usepackage{array}
\usepackage{setspace}
\usepackage{mathtools}
\usepackage{amsmath,amssymb}
\usepackage{algorithm,algorithmic}
\usepackage{caption}
\usepackage{cleveref}
\usepackage{titlesec}
\usepackage{multirow}
\usepackage{subfigure}
\usepackage{mathrsfs}
\usepackage{colortbl}
\usepackage{todonotes}
\usepackage[utf8]{inputenc}
\usepackage{pgfplots}
\DeclareUnicodeCharacter{2212}{−}
\usepgfplotslibrary{groupplots,dateplot}
\usetikzlibrary{patterns,shapes.arrows}
\pgfplotsset{compat=newest}

\titleformat*{\subparagraph}{\itshape}

\newtheorem{proposition}{Proposition}

\numberwithin{equation}{section}

\definecolor{babypink}{rgb}{0.96, 0.76, 0.76}
\definecolor{burntsienna}{rgb}{0.91, 0.45, 0.32}     
\definecolor{crimson}{rgb}{0.86, 0.08, 0.24}
\definecolor{darkspringgreen}{rgb}{0.09, 0.45, 0.27}
\definecolor{highlighColor}{rgb}{0.91, 0.41, 0.17}  
\definecolor{royalblue}{rgb}{0.25, 0.41, 0.88}
\definecolor{royalpurple}{rgb}{0.47, 0.32, 0.66}
\definecolor{ruddy}{rgb}{1.0, 0.0, 0.16}
\definecolor{deepcarrotorange}{rgb}{0.91, 0.41, 0.17}
\definecolor{darkspringgreen}{rgb}{0.09, 0.45, 0.27}

\definecolor{rebuttalColor}{rgb}{0.17,0.63,0.17}
\newcommand{\rebuttal}[1]{{#1}}

\def\rset{\mathbb{R}}

\def\rmd{\mathrm{d}}

\def\normpdf{\mathrm{N}}
\def\eqsp{\,}
\def\bigO{\mathcal{O}}
\def\Id{\mathrm{I}}
\def\zero{0}

\def\pE{\mathbb{E}}

\def\eqdef{\vcentcolon=}

\def\wrt{w.r.t.}
\newcommand{\tbar}{\mathrel{\raisebox{0.15ex}{$\scriptscriptstyle\mid$}}}

\def\law{\mathsf{Law}}
\def\algo{{\sc{DInG}}}

\def\Tf{1}

\def\mask{\mathbf{m}}

\def\unmask{{\overline{\mathbf{m}}}}

\newcommand{\hilight}[1]{\textcolor{ruddy}{#1}}
\newcommand{\scdhilight}[1]{\textcolor{darkspringgreen}{#1}}

\newcommand{\tk}[1]{{t_{#1}}}

 \def\rset{{\mathbb{R}}}

 \def\rmd{\mathrm{d}}
 \def\eqsp{\;}
 \def\eqdef{\coloneqq}

\newcommandx\targ[4][4=]{
    \ifthenelse{\equal{#2}{}}{
        \ifthenelse{\equal{#3}{}}{
            p^{#4} _{#1}
        }{
            p^{#4} _{#1}(#3)
            }
    }{
        \ifthenelse{\equal{#3}{}}{
            p^{#4} _{#1}(\cdot|#2)
        }{
            p^{#4} _{#1}(#3|#2)
        }
    }}

\newcommandx\interp[4][4=\interpscale]{
    \ifthenelse{\equal{#2}{}}{
        \ifthenelse{\equal{#3}{}}{
            \pi^{#4} _{#1}
        }{
            \pi^{#4} _{#1}(#3)
            }
    }{
        \ifthenelse{\equal{#3}{}}{
            \pi^{#4} _{#1}(\cdot|#2)
        }{
            \pi^{#4} _{#1}(#3|#2)
        }
    }}

\newcommandx\pdata[4][4=]{
    \ifthenelse{\equal{#2}{}}{
        \ifthenelse{\equal{#3}{}}{
            p^{#4} _{#1}
        }{
            p^{#4} _{#1}(#3)
        }
    }{
        \ifthenelse{\equal{#3}{}}{
            p^{#4} _{#1}(\cdot|#2)
        }{
            p^{#4} _{#1}(#3|#2)
        }
    }}

\newcommandx\hpdata[4][4=]{
    \ifthenelse{\equal{#2}{}}{
        \ifthenelse{\equal{#3}{}}{
            \hat{p}^{#4} _{#1}
        }{
            \hat{p}^{#4} _{#1}(#3)
        }
    }{
        \ifthenelse{\equal{#3}{}}{
            \hat{p}^{#4} _{#1}(\cdot|#2)
        }{
            \hat{p}^{#4} _{#1}(#3|#2)
        }
    }}

\newcommandx\revker[4][4=]{
\ifthenelse{\equal{#2}{}}{
    \ifthenelse{\equal{#3}{}}{
        r^{#4} _{#1}
    }{
        r^{#4} _{#1}(#3)
    }
}{
    \ifthenelse{\equal{#3}{}}{
        r^{#4} _{#1}(\cdot|#2)
    }{
        r^{#4} _{#1}(#3|#2)
    }
}}

\newcommandx\hatrevker[4][4=]{
    \ifthenelse{\equal{#2}{}}{
        \ifthenelse{\equal{#3}{}}{
            \hat{r}^{#4} _{#1}
        }{
            \hat{r}^{#4} _{#1}(#3)
        }
    }{
        \ifthenelse{\equal{#3}{}}{
            \hat{r}^{#4} _{#1}(\cdot|#2)
        }{
            \hat{r}^{#4} _{#1}(#3|#2)
        }
    }
}

\newcommandx\refreshker[4][4=]{
    \ifthenelse{\equal{#2}{}}{
        \ifthenelse{\equal{#3}{}}{
            m^{#4} _{#1}
        }{
            m^{#4} _{#1}(#3)
        }
    }{
        \ifthenelse{\equal{#3}{}}{
            m^{#4} _{#1}(\cdot|#2)
        }{
            m^{#4} _{#1}(#3|#2)
        }
    }}
\newcommand\jpdata[3]{
    \ifthenelse{\equal{#2}{}}{
        \ifthenelse{\equal{#3}{}}{
            \bar{p}_{#1}
        }{
            \bar{p}_{#1}(#3)
        }
    }{
        \ifthenelse{\equal{#3}{}}{
            \bar{p}_{#1}(\cdot|#2)
        }{
            \bar{p}_{#1}(#3|#2)
        }
    }}

\newcommandx\post[4][4=]{
    \ifthenelse{\equal{#2}{}}{
        \ifthenelse{\equal{#3}{}}{
            \pi^{#4} _{#1}
        }{
            \pi^{#4} _{#1}(#3)
        }
    }{
        \ifthenelse{\equal{#3}{}}{
            \pi^{#4} _{#1}(\cdot|#2)
        }{
            \pi^{#4} _{#1}(#3|#2)
        }
    }}

\newcommandx\hpost[4][4=]{
    \ifthenelse{\equal{#2}{}}{
        \ifthenelse{\equal{#3}{}}{
            \hat\pi^{#4} _{#1}
        }{
            \hat\pi^{#4} _{#1}(#3)
        }
    }{
        \ifthenelse{\equal{#3}{}}{
            \hat\pi^{#4} _{#1}(\cdot|#2)
        }{
            \hat\pi^{#4} _{#1}(#3|#2)
        }
    }}

  \newcommandx\pot[3][3=]{
        \ifthenelse{\equal{#3}{}}{
            \ell^{#3} _{#1}(\obs|#2)
        }{
            \ell^{#3} _{#1}(\obs|#2)
        }
    }

\newcommandx\hpot[4][4=]{
    \ifthenelse{\equal{#2}{}}{
        \ifthenelse{\equal{#3}{}}{
            \hat{\ell}^{#4} _{#1}
        }{
            \hat{\ell}^{#4} _{#1}(#3)
        }
    }{
        \ifthenelse{\equal{#3}{}}{
            \hat{\ell}^{#4} _{#1}(\cdot|#2)
        }{
            \hat{\ell}^{#4} _{#1}(#3|#2)
        }
    }}

\newcommandx\fw[4][4=]{
        \ifthenelse{\equal{#3}{}}{
            q^{#4} _{\smash{#1}}(\cdot|#2)
        }{
            q^{#4} _{\smash{#1}}(#3|#2)
        }
    }

\newcommandx\denoiser[5][4=, 5=0]{
    \ifthenelse{\equal{#2}{}}{
        \ifthenelse{\equal{#3}{}}{
            \hat\bx^{#4}_{#5}(\cdot, #1)
        }{
            \hat\bx^{#4} _{#5}(#3, #1)
        }
    }{
        \ifthenelse{\equal{#3}{}}{
            \hat\bx^{#4} _{#5}(\cdot, #1|#2)
        }{
            \hat\bx^{#4} _{#5}(#3, #1|#2)
        }
    }
}

\newcommandx\noisepred[4][4=]{
    \ifthenelse{\equal{#2}{}}{
        \ifthenelse{\equal{#3}{}}{
            \hat\bx^{#4}_{1}(\cdot, #1)
        }{
            \hat\bx^{#4} _{1}(#3, #1)
        }
    }{
        \ifthenelse{\equal{#3}{}}{
            \hat\bx^{#4} _{1}(\cdot, #1|#2)
        }{
            \hat\bx^{#4} _{1}(#3, #1|#2)
        }
    }
}

    \newcommandx\hdenoiser[4][4=]{
    \ifthenelse{\equal{#2}{}}{
        \ifthenelse{\equal{#3}{}}{
            \hat{D}^{#4}_{#1}
        }{
            \hat{D}^{#4} _{#1}(#3)
        }
    }{
        \ifthenelse{\equal{#3}{}}{
            \hat{D}^{#4} _{#1}(\cdot|#2)
        }{
            \hat{D}^{#4} _{#1}(#3|#2)
        }
    }}

\newcommandx\epspred[4][4=]{
    \ifthenelse{\equal{#2}{}}{
        \ifthenelse{\equal{#3}{}}{
            \varepsilon^{#4}_{#1}
        }{
            \varepsilon^{#4} _{#1}(#3)
        }
    }{
        \ifthenelse{\equal{#3}{}}{
            \varepsilon^{#4} _{#1}(\cdot|#2)
        }{
            \varepsilon^{#4} _{#1}(#3|#2)
        }
    }}

\newcommand\clf[3]{
    \ifthenelse{\equal{#2}{}}{
        g_{#1}(#3|\cdot)
    }{
        g _{#1}(#3|#2)
    }
}

\newcommand\cfgdist[3]{
    \ifthenelse{\equal{#2}{}}{
        p^w _{#1}(#3|\cdot)
    }{
        p^w _{#1}(#3|#2)
    }
}

\newcommand\cscore[3]{
    \ifthenelse{\equal{#2}{}}{
        \ifthenelse{\equal{#3}{}}{
            s _{#1}
        }{
            s _{#1}(#3)
        }
    }{
        \ifthenelse{\equal{#3}{}}{
            s _{#1}(\cdot|#2)
        }{
            s _{#1}(#3|#2)
        }
    }}

\def\gauss{\mathcal{N}}

\def\interpscale{\eta}

\def\std{\sigma}

\def\acp{\alpha}
\def\ddimstd{\eta}

\def\bx{\mathbf{x}}
\def\bv{\mathbf{v}}

\def\bz{\mathbf{z}}

\def\bmu{\bm{\mu}}

\def\bw{\mathbf{w}}

\def\blended{{\sc{Blended-Diff}}}
\def\dps{\mathsf{dps}}
\def\ding{\mathsf{ding}}
\def\ddnm{{\sc{DDNM}}}
\def\diffpir{{\sc{DiffPIR}}}
\def\flowdps{{\sc{FlowDPS}}}
\def\flowchef{{\sc{FlowChef}}}
\def\reddiff{{\sc{RedDiff}}}
\def\daps{{\sc{DAPS}}}
\def\resample{{\sc{ReSample}}}
\def\psld{{\sc{PSLD}}}
\def\pnpflow{{\sc{PnP-Flow}}}

\def\ffhq{{\texttt{FFHQ}}}
\def\div2k{{\texttt{DIV2K}}}
\def\PieBench{{\texttt{PIE-Bench}}}

\def\param{\theta}

\def\dimobs{{d_\obs}}
\def\dimx{{d}}

\def\obs{\mathbf{y}}

\newcommandx{\hpredx}[3][2=0,3=\param]{\smash{m^{#3} _{#2|#1}}}
\newcommandx{\predx}[2][2=0]{\smash{m _{#2|#1}}}
\newcommandx{\prednoise}[2][2=\param]{\smash{\epsilon^{#2} _{#1}}}
\newcommandx{\score}[2][2=\param]{s^{#2} _{#1}}

\def\encoder{\mathcal{E}}

\newcommand{\tcp}[1]{\textcolor{purple}{#1}}

\newcommand{\tcz}[1]{\textcolor{ruddy}{#1}}
\def\encoder{\mathsf{Enc}}

\newcommand\blfootnote[1]{%
  \begingroup
  \renewcommand\thefootnote{}\footnote{#1}%
  \addtocounter{footnote}{-1}%
  \endgroup
}

\newcounter{hypA}

\title{Efficient Zero-Shot Inpainting with Decoupled Diffusion Guidance}

\iclrfinalcopy

\author{%
Badr Moufad\textsuperscript{1,*} \,
Navid Bagheri Shouraki\textsuperscript{1,5,7} \,
Alain Oliviero Durmus\textsuperscript{1} \,
\\
{\bf{%
Thomas Hirtz\textsuperscript{6} \,
Eric Moulines\textsuperscript{3,4} \,
Jimmy Olsson\textsuperscript{8} \,
Yazid Janati\textsuperscript{2,3,*} \,
}}
\\
\\
\textsuperscript{1}CMAP, Ecole Polytechnique \,
\textsuperscript{2}Institute of Foundation Models \,
\textsuperscript{3}MBZUAI \,
\textsuperscript{4}EPITA \,
\\
\textsuperscript{5}Sorbonne University \,
\textsuperscript{6}Lagrange Mathematics and Computing Research Center \,
\\
\textsuperscript{7}EPITA Research Lab \,
\textsuperscript{8}KTH Royal Institute of Technology \,
}

\begin{document}
\maketitle

\begin{abstract}
Diffusion models have emerged as powerful priors for image editing tasks such as inpainting and local modification, where the objective is to generate realistic content that remains consistent with observed regions. In particular, zero-shot approaches that leverage a pretrained diffusion model, without any retraining, have been shown to achieve highly effective reconstructions. However, state-of-the-art zero-shot methods typically rely on a sequence of surrogate likelihood functions, whose scores are used as proxies for the ideal score. This procedure however requires vector-Jacobian products through the denoiser at every reverse step, introducing significant memory and runtime overhead. To address this issue, we propose a new likelihood surrogate that yields simple and efficient to sample Gaussian posterior transitions, sidestepping the  backpropagation through the denoiser network. Our extensive experiments show that our method achieves strong observation consistency compared with fine-tuned baselines and produces coherent, high-quality reconstructions, all while significantly reducing inference cost.\\
Code is available at \url{https://github.com/YazidJanati/ding}.
\end{abstract}
\blfootnote{*Authors contributed equally}
\blfootnote{\ Correspondence: $\{ \texttt{badr.moufad@polytechnique.edu} \}$, $\{ \texttt{yazid.janati@mbzuai.ac.ae} \}$}

\section{Introduction}
\label{sec:intro}
 We focus on \emph{inpainting problems} in computer vision, which play a central role in applications ranging from photo restoration to content creation and interactive design. Given an image with prescribed missing pixels, the objective is to generate a semantically coherent completion while ensuring strict consistency with the observed region. The importance of this task has motivated extensive research, spanning both classical approaches and, more recently, generative modeling with diffusion models \citep{rombach2022high, esser2024scaling, batifol2025flux, wu2025qwen}.

To address this problem, two main diffusion-based approaches have been popularized.  
The first relies on training \emph{conditional diffusion models} tailored to a specific editing setup. These models directly approximate the conditional distribution of interest \citep{saharia2022palette, wang2023imagen, kawar2023imagic, huang2025diffusion} and take as side inputs additional information such as a mask, a text prompt, or reference pixels \citep{saharia2022palette, wang2023imagen, kawar2023imagic, huang2025diffusion}. An alternative approach, which has recently attracted growing attention, is \emph{zero-shot image editing}, requiring no extra training or fine-tuning. In this formulation, the task is cast as a Bayesian inverse problem: the pre-trained diffusion model serves as a prior, while a likelihood term enforces fidelity to the observations, and the resulting posterior distribution defines the reconstructions \citep{song2019generative, song2021score, kadkhodaie2020solving, kawar2022denoising, lugmayr2022repaint, avrahami2022blended, chung2023diffusion, mardani2024a, rout2024semantic}. Sampling from this posterior is achieved by approximating the score functions associated with the diffusion model adapted to this distribution.  This \emph{plug-and-play} paradigm has been investigated across a variety of inverse problems, from image restoration to scientific imaging, and has demonstrated strong editing performance without task-specific training.

While current zero-shot methods are appealing, they face a critical practical limitation. Implementations of strong zero-shot posterior sampling with diffusion priors typically rely on the twisting function proposed by \citet{ho2022video, chung2023diffusion, song2022pseudoinverse}, which corresponds to the  likelihood evaluated at the denoiser's output given the observation. Simulating the corresponding reverse diffusion process requires computing  gradients of the denoiser with respect to its input. This in turn entails repeated backpropagation through the denoiser network and costly vector–Jacobian product (VJP) evaluations. This makes such methods computationally demanding, memory intensive, and often slower than training a dedicated conditional model.

\paragraph{Contributions.} 
We propose a new \emph{VJP-free} framework for zero-shot inpainting with a pre-trained diffusion prior. Our key idea is to approximate the intractable  twisted posterior-sampling  transitions by a closed-form mixture distribution that can be sampled exactly, thereby eliminating the need for VJP evaluations and backpropagation through the denoiser. Concretely, we modify the twisting function of \citet{ho2022video,chung2023diffusion} so that it evaluates the denoiser at an independent draw from the pretrained transition. This decoupling breaks the dependency between the denoiser and the arguments of the transition density. As a result,
our method provides posterior transitions that can be sampled efficiently for zero-shot inpainting with latent diffusion models. We demonstrate through extensive experiments on Stable Diffusion (SD) 3.5 that our method, coined {\sc{Decoupled INpainting Guidance}} (\algo),  consistently outperforms state-of-the-art guidance methods under low NFE budgets. It achieves, across three benchmarks, the best trade-off between fidelity to the visible content and realism of the reconstructions, while being both faster and more memory-efficient than competing approaches. Remarkably, even without any task-specific fine-tuning, it outperforms an SD~3 model that has been  fine-tuned for image editing, confirming the effectiveness and practicality of our framework.

\begin{figure}[t]
    \centering
    \includegraphics[width=\textwidth]{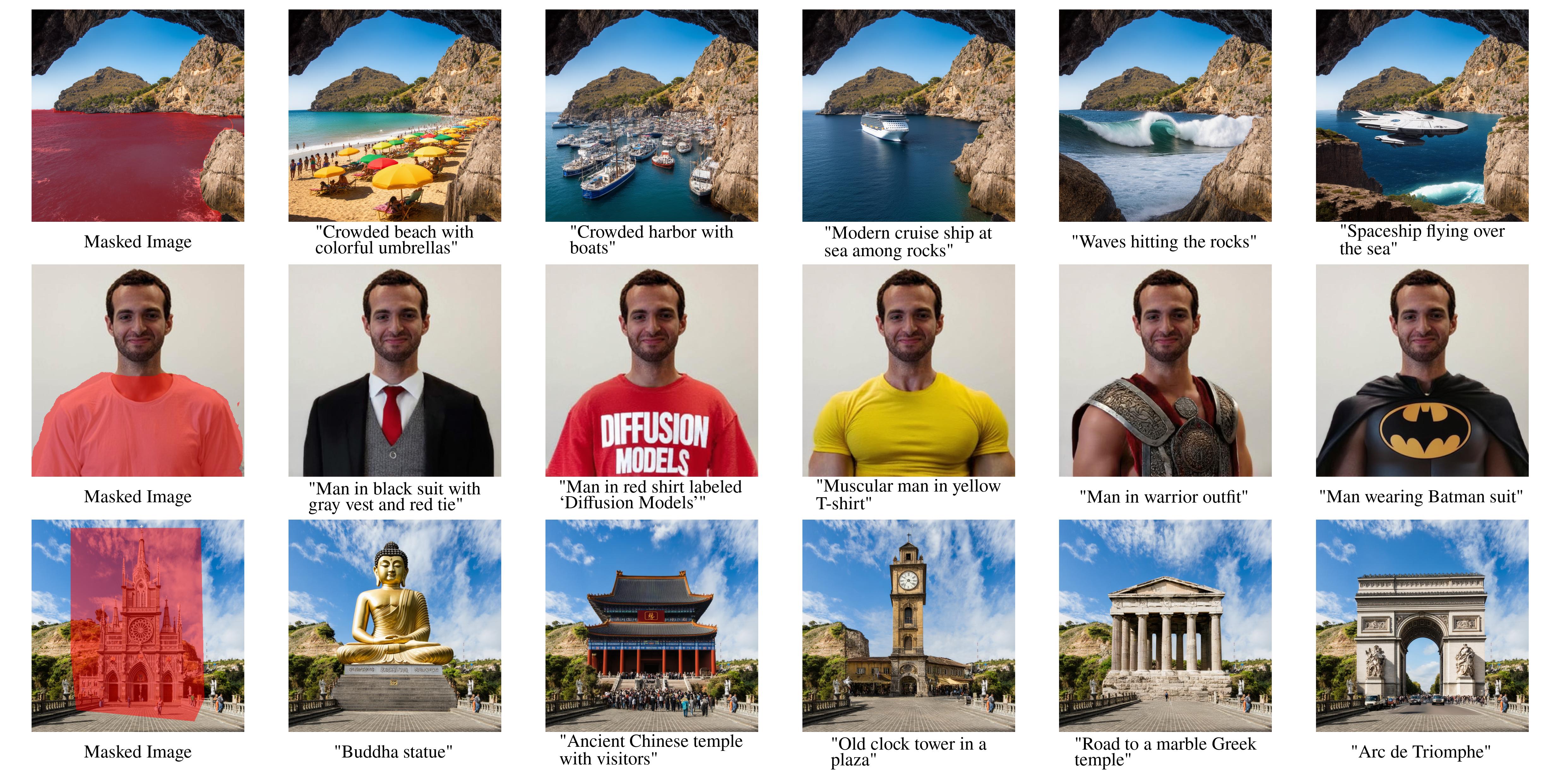}
    \captionsetup{font=small}
    \caption{Zero-shot inpainting edits generated by \algo\ (50 NFEs) for different masking patterns using Stable Diffusion 3.5 (medium). Given masked inputs (left column), the model fills the missing regions according to diverse textual prompts.}
    \label{fig:Image editing DIV2K}
\end{figure}

\section{Background}
\paragraph{Diffusion models}
Denoising diffusion models (DDMs) \citep{sohl2015deep, song2019generative, ho2020denoising} define a generative process for a data distribution $\pdata{0}{}{}$ by constructing a continuous path $(\pdata{t}{}{})_{t \in [0,1]}$ of  distributions between $p_0$ and 
$p_1 \eqdef \gauss(\zero, \Id_\dimx)$. More precisely, $\pdata{t}{}{} = \law(X_t)$, where
\begin{equation}
\label{eq:interpolation}
X_t = \alpha_t X_0 + \sigma_t X_1 \eqsp, 
\quad X_0 \sim \pdata{0}{}{} \eqsp, 
\quad X_1 \sim \pdata{1}{}{} \eqsp.
\end{equation}
Here $X_0$ and $X_1$ are supposed to be independent and $(\alpha_t)_{t \in [0,1]}$ and $(\sigma_t)_{t \in [0,1]}$ are deterministic, non-increasing and non-decreasing, respectively, schedules with boundary conditions $(\alpha_0,\sigma_0) \eqdef (1,0)$ and $(\alpha_1,\sigma_1) \eqdef (0,1)$. Typical choices include the \emph{variance-preserving schedule}, satisfying $\alpha_t^2 + \sigma_t^2 = 1$ \citep{ho2020denoising, dhariwal2021diffusion}, and the \emph{linear schedule}, defined by $(\alpha_t,\sigma_t) = (1-t, t)$ \citep{lipman2023flow, esser2024scaling, gao2024diffusion}. The path $(\pdata{t}{}{})_{t \in [0,1]}$ defines an interpolation that gradually transforms the clean data distribution $\pdata{0}{}{}$ into the Gaussian reference distribution $\pdata{1}{}{}$. 
To generate new samples, DDMs simulate a time-reversed Markov chain. Given a decreasing sequence $(\tk{k})_{k=0}^K$ of time steps with $\tk{K}=1$ and $\tk{0}=0$, reverse transitions are iteratively applied to map a sample from $\pdata{\tk{k+1}}{}{}$ to one from $\pdata{\tk{k}}{}{}$, thereby progressively denoising until convergence to the clean distribution $\pdata{0}{}{}$.

The DDIM framework \citep{song2021ddim} introduces a general family of reverse transitions for denoising diffusion models. It defines a new schedule $(\eta_t)_{t \in [0,1]}$, satisfying $\eta_t \leq \sigma_t$ for all $t \in [0,1]$, along with a family of transition densities given for $s < t$ by
\begin{equation}
\label{eq:ddpm-reverse}
 \pdata{s|t}{\bx_{t}}{\bx_s}[\ddimstd]
 = \pE \Big[\, \fw{s\tbar 0, 1}{X_0, X_1}{\bx_s}[\ddimstd] \,\Big|\, X_t = \bx_t \Big] \eqsp,
\end{equation}
where $\fw{s\tbar 0, 1}{\bx_0, \bx_1}{\bx_s}[\ddimstd] 
\eqdef \normpdf(\bx_s;\, \alpha_s \bx_0 + \sqrt{\sigma_s^2 - \eta_s^2}\, \bx_1,\, \eta_s^2 \Id)$ and the random variables $(X_0, X_t, X_1)$ are defined as in \eqref{eq:interpolation}.   By construction, this family satisfies the marginalization property 
$
\smash{\pdata{s}{}{\bx_s}
= \int \pdata{s\tbar t}{\bx_t}{\bx_s}[\ddimstd] \, \pdata{t}{}{\bx_t}\, \mathrm d \bx_t}
$ \cite[Appendix B]{song2021ddim}. 
Thus, $(p_{t_k\tbar t_{k + 1}}^\eta)_{k = 0}^{K-1}$ defines a consistent set of reverse transitions, enabling stepwise sampling from the sequence $(p_{t_k})_{k=0}^K$.  In practice, however, these transitions are intractable. A common approximation is to replace $X_0$ and $X_1$ in \eqref{eq:ddpm-reverse} by their conditional expectations \citep{ho2020denoising, song2021ddim}. More precisely, let $\denoiser{t}{}{}[\param]$ denote a parametric estimator of
$
\denoiser{t}{}{\bx_t} \eqdef \pE[X_0 \mid X_t = \bx_t]
$.
Since $\pE[X_1 | X_t = \bx_t] = (\bx_t - \acp_t \denoiser{t}{}{\bx_t}) / \std_t$, we set $\denoiser{t}{}{\bx_t}[\param][1] \eqdef (\bx_t - \acp_t \denoiser{t}{}{\bx_t}[\param][0]) / \std_t$. Then the parametric model proposed by \cite{ho2020denoising,song2021ddim} corresponds to approximating 
each $p_{t_k \tbar t_{k + 1}}^\eta$ by 
\begin{equation}
    \label{eq:ddpm-approximation}
\pdata{\tk{k} \tbar \tk{k+1}}{\bx_{\tk{k+1}}}{\bx_\tk{k}}[\ddimstd, \param] \eqdef \fw{\tk{k}\tbar 0,1}{\denoiser{\tk{k+1}}{}{\bx_\tk{k+1}}[\param], \denoiser{\tk{k+1}}{}{\bx_\tk{k+1}}[\param][1]}{\bx_\tk{k}}[\ddimstd]  \eqsp.
\end{equation}
For $k = 0$, $\pdata{{0}\tbar \tk{1}}{\bx_\tk{1}}{}[\ddimstd, \param]$ is simply defined as the Dirac mass at $\denoiser{\tk{1}}{}{\bx_\tk{1}}[\param]$. In the rest of the paper we omit the superscript $\eta$ when there is no ambiguity.
\paragraph{Image editing.}
In this work, we address the task of image editing via inpainting. We assume access to some reference image $\bx_\ast \in \rset^\dimx$ that must be modified while remaining consistent with a prescribed set of observed pixels. Let 
$\mask \subset \{1,\ldots, \dimx\}$
denote the index set of missing (masked) pixels, and let 
$\unmask = \{1, \ldots, \dimx\} \setminus \mask$ 
be the index set of observed (unmasked) pixels, with cardinality $|\unmask| = \dimobs$.  
For any $\bx \in \rset^\dimx$ and 
$\mathbf{i} \subset \{1, \ldots, \dimx\}$, we denote by $\bx[\mathbf{i}] \in \rset^{|\mathbf{i}|}$ the subvector formed by the components of $\bx$ with indices $\mathbf{i}$. The observation is thus given by $\obs \eqdef \bx_\ast[\unmask]$, and the objective is to synthesize a reconstruction $\hat\bx$ such that
$\hat\bx[\unmask] \approx \obs$ while generating the missing region $\hat\bx[\mask]$ in a realistic and semantically coherent manner with respect to the observed pixels. In the Bayesian formulation, the data distribution $p_0$ serves as a prior over natural images, while the observation model is encoded by a Gaussian likelihood on the observed coordinates:
\begin{equation}
 \label{eq:inpainting-likelihood}
 \pot{0}{\bx} = \normpdf\!\left(\obs ; \bx[\unmask], \, \std_\obs^2 \, \Id_{\dimobs}\right).
\end{equation}
The parameter $\std_\obs > 0$ serves as a relaxation factor: smaller values enforce strict adherence to the observation, while larger values permit controlled deviations from $\bx_\ast$, thereby facilitating the reconstruction process. In this Bayesian framework, the target distribution from which we aim to sample is the \emph{posterior distribution}
\begin{equation}
 \label{eq:posterior}
 \post{0}{\obs}{\bx_0} \propto \pot{0}{\bx_0}\,\pdata{0}{}{\bx_0} \eqsp.
\end{equation}
\paragraph{Inference-time guidance.}
As observed in the seminal works of \citet{song2019generative,kadkhodaie2020solving,song2021score,kawar2021snips}, approximate sampling from the posterior distribution can be performed by biasing the denoising process with guidance terms, without requiring any additional fine-tuning.
The central idea is to modify the sampling dynamics of diffusion models on-the-fly so that the generated samples both satisfy the likelihood constraint $\pot{0}{\cdot}$ and remain plausible under the prior $\pdata{0}{}{}$. More precisely, a standard approach is to approximate the iterative updates of a diffusion model defined to target the posterior $\post{0}{\obs}{\cdot}$.
This in turn entails deriving an approximation of the posterior denoiser $\denoiser{t}{\obs}{\bx_t} \eqdef \int \bx_0 \, \post{0\tbar t}{\bx_t, \obs}{\bx_0} \, \rmd \bx_0$, where $\post{0\tbar t}{\bx_t, \obs}{\bx_0} \propto \post{0}{\obs}{\bx_0} \normpdf(\bx_t; \acp_t \bx_0, \std^2 _t \Id)$. The denoiser $\denoiser{t}{\obs}{}$ is related to the prior denoiser via the identity 
\begin{equation} 
    \label{eq:posterior-denoiser}
    \denoiser{t}{\obs}{\bx_t} = \denoiser{t}{}{\bx_t} + \acp^{-1} _t \std^2 _t \nabla_{\bx_t} \log \pot{t}{\bx_t} \eqsp,
\end{equation} 
where the additional term is referred to as the \emph{guidance term}; see \citet[Eq. 2.15 and 2.17]{daras2024survey}.  It is defined as the  logarithmic gradient of the \emph{propagated likelihood}
\begin{equation}
\label{eq:guidance-term}
\pot{t}{\bx_t} 
\eqdef \int \pot{0}{\bx_0} \, \pdata{0\tbar t}{\bx_t}{\bx_0}\, \rmd \bx_0 \eqsp, \text{ with } \, \pdata{0\tbar t}{\bx_t}{\bx_0} \propto \pdata{0}{}{\bx_0}\,\normpdf(\bx_t;\, \alpha_t \bx_0,\, \sigma_t^2 \Id) \eqsp;
\end{equation}
see \citet[Equation 2.20]{daras2024survey}.    
Since the pre-trained parametric approximation $\denoiser{t}{}{}[\param]$ of the prior denoiser $\denoiser{t}{}{}$ is already available, estimating $\denoiser{t}{\obs}{}$ reduces to computing the intractable score term $\nabla_{\bx_t} \log \pot{t}{\bx_t}$. A widely adopted approximation \citep{ho2022video, chung2023diffusion} replaces $\pdata{0\tbar t}{\bx_t}{\cdot}$ in \eqref{eq:guidance-term} by a Dirac mass at the denoiser estimate $\denoiser{t}{}{\bx_t}[\param]$, yielding
\begin{equation}
\label{eq:dps}
\hpot{t}{\bx_t}{\obs}[\param] 
\eqdef \pot{0}{\denoiser{t}{}{\bx_t}[\param]} \eqsp.
\end{equation}
This approximation is often combined with a suitable rescaling weight (possibly depending on $\bx_t$); see \citet[Equation 8]{ho2022video} and \citet[Algorithm 1]{chung2023diffusion}. Substituting this into the identity \eqref{eq:posterior-denoiser} yields an approximation of the posterior denoiser, which in turn defines an approximate diffusion model for $\post{0}{\obs}{}$.

\section{Method}
The methods discussed in the previous section rely on the likelihood approximation \eqref{eq:dps}, which is then inserted into \eqref{eq:posterior-denoiser}. However, computing this term requires differentiating through the denoisers $\denoiser{t_k}{}{}[\param]$ at each timestep $t_k$. This operation is computationally demanding: it increases memory usage, slows down the sampling process, and reduces scalability.  By contrast, fine-tuned conditional diffusion models bypass these inference costs once training is complete, but at the expense of per-task retraining. This highlights a fundamental trade-off: zero-shot posterior sampling eliminates the need for retraining, but incurs substantial overhead during inference. 
Our goal is to bridge this gap by designing a zero-shot posterior sampler that removes the need for backpropagation through the denoiser while preserving the effectiveness of guidance.
\begin{algorithm}[t]
    \caption{Posterior sampling with decoupled guidance}
    \begin{algorithmic}[1]
       \STATE {\bfseries Input:} decreasing timesteps $(t_k)_{k=K}^0$ with $t_K = 1$, $t_0 = 0$; original image $\bx_\ast$; mask $\mask$; \\
       DDIM parameters $(\eta_k)_{k = K} ^0$.
       \vspace{.1cm}\\
       \STATE $\obs \gets \bx_\ast[\unmask]; \quad \bx \sim \gauss(\zero, \Id_\dimx)$
       \FOR{$k = K-1$ {\bfseries to} $1$}
            \STATE $\hat\bx_0 \gets \bx^\param _0(\bx, \tk{k+1})$
            \STATE $\hat\bx_1 \gets (\bx - \acp_\tk{k+1} \hat\bx_0) / \std_\tk{k+1}$
            \STATE $\bmu \gets \acp_\tk{k} \hat\bx_0 + (\std^2 _\tk{k} - \ddimstd^2 _{k})^{1/2} \hat\bx_1$ \\
            \vspace{.1cm}
            \textcolor{blue}{/* \texttt{Sampling \eqref{eq:gaussian-conjugacy}} */}
            \STATE $(\bw, \bw^\prime) \overset{\tiny{\mbox{i.i.d.}}}{\sim} \gauss(\zero_\dimx, \Id_\dimx)$\\
            \STATE $\tcz{\bz} \gets \bmu + \ddimstd _k \bw$
            \STATE $\hat\bx^{\mathrm{pxy}} _1 \gets (\tcz{\bz} - \acp_\tk{k} \denoiser{\tk{k}}{}{\tcz{\bz}}[\param]) / \std_\tk{k}$ \label{line:noise_proxy}
            \STATE $\gamma \gets  \eta^2 _\tk{k} / (\eta^2 _\tk{k} + \acp^2 _\tk{k} \std^2 _\obs)$
            \STATE $\bx[\mask] \gets \bmu[\mask] + \ddimstd _k \bw^\prime[\mask]$ \label{line:masked_update}
            \STATE $\bx[\unmask] \gets (1 - \gamma) \bmu[\unmask] + \gamma \big( \acp_\tk{k} \obs + \std_\tk{k} \hat\bx^{\mathrm{pxy}} _1[\unmask] \big) + \acp_\tk{k} \std _\obs \sqrt{\gamma}  \bw^\prime [\unmask]$ \label{line:unmasked_update}
       \ENDFOR
       \STATE {\bfseries Return:} $\denoiser{\tk{1}}{}{\bx}[\param]$
    \end{algorithmic}
    \label{algo:decoupled}
\end{algorithm}

\paragraph{Reverse transitions for the posterior.}  
Our method builds upon the alternative sampling strategy introduced in \cite{wu2023practical,zhang2023towards,janati2024dcps}. 
Instead of initializing the interpolation \eqref{eq:interpolation} with the prior $X_0 \sim \pdata{0}{}{}$, we consider the same process initialized from the posterior distribution $X_0 \sim \post{0}{\obs}{}$. This yields a new family of random variables whose marginals are
$\post{t}{\obs}{\bx_t} 
\eqdef \int \normpdf(\bx_t;\, \alpha_t \bx_0,\, \sigma_t^2 \Id_\dimx)\, \post{0}{\obs}{\bx_0}\, \rmd \bx_0$, 
in analogy with the prior family $(\pdata{t}{}{})_{t \in [0,1]}$. Moreover, the DDIM transitions associated with $(\post{t}{\obs}{})_{t \in [0,1]}$ are given by \citet[Equation~1.17]{janati2025bridging}:
\begin{equation}
\label{eq:posterior-transition}
\post{s\tbar t}{\bx_{t}, \obs}{\bx_s}[\ddimstd] 
\propto \pot{s}{\bx_s}\, \pdata{s|t}{\bx_{t}}{\bx_s}[\ddimstd] \eqsp,
\end{equation}
which defines a valid Markov chain with marginals $(\post{t_k}{\obs}{})_{k=1}^K$. This chain defines a  path between the Gaussian reference $\gauss(\zero,\Id_\dimx)$ and the posterior distribution $\post{0}{\obs}{}$. However, the presence of the likelihood term $\pot{t}{\bx_t}$ makes also these transitions intractable. To address this issue, prior works \citep{zhang2023towards, wu2023practical} introduced the surrogate transitions proportional to $\bx_s \mapsto \hpot{s}{\bx_s}{\obs}[\param] \pdata{s\tbar t}{\bx_{t}}{\bx_s}[\eta,\param]$, for fixed $\bx_t$ and $\obs$,
where $\hpot{t}{\cdot}{\obs}[\param]$ are defined in \eqref{eq:dps}. These transitions  are then approximated using either variational inference \citep{janati2024dcps,pandey2025variational} or  
 sequential Monte Carlo methods \citep{wu2023practical}. However, similar to the methods described in the previous section, these approximations rely on the approximate guidance term and thus suffer from inflated memory usage and higher runtime. 
 
\paragraph{Our likelihood approximation.}  
To address this limitation, we draw inspiration from \eqref{eq:dps} to propose a lightweight approximation, designed to eliminate the need for VJP evaluations through the denoiser. Using the relation $\denoiser{s}{}{\bx_s}[\param][0]  =(\bx_s-\sigma_s \denoiser{s}{}{\bx_s}[\param][\Tf])/\alpha_s$, we first rewrite 
the standard likelihood approximation \eqref{eq:dps} in terms of the noise prediction $\denoiser{s}{}{\bx_s}[\param][\Tf]$ according to 
\[
\hpot{s}{\bx_s}{\obs}[\param] 
= \pot{0}{(\bx_s - \sigma_s \denoiser{s}{}{\bx_s}[\param][\Tf])/\alpha_s} 
\eqsp.
\]
Based on this parametrization, we then introduce the following alternative approximation
\begin{equation}
\label{eq:approx-decoupled-potential}
\hpot{s}{\bx_s, \tcz{\bz_s}}{\obs}[\param] 
\eqdef \pot{0}{(\bx_s - \sigma_s \denoiser{s}{}{\tcz{\bz_s}}[\param][\Tf])/\alpha_s} \eqsp,
\end{equation}
where the noise predictor is evaluated at $\tcz{\bz_s} \in \rset^d$, which serves as a proxy for $\bx_s$. 
A key feature of this decoupling is that it 
enables lightweight updates, avoids costly denoiser backpropagation, and still provides high-quality reconstructions.
Then, similarly to \eqref{eq:posterior-transition}, we define 
\[
\hpost{s\tbar t}{\tcz{\bz_s}, \bx_t, \obs}{\bx_s}[\param] 
\propto \hpot{s}{\bx_s, \tcz{\bz_s}}{\obs}[\param] \; \pdata{s\tbar t}{\bx_t}{\bx_s}[\eta,\param] \eqsp.
\]
This leads us to propose the surrogate\footnote{In the follow-up work \citet{ghorbel2026ding-editor}, we provide further insight into this approximation.}
\begin{equation}
\label{eq:algo-transition}
\hpost{s\tbar t}{\bx_t, \obs}{\bx_s}[\param] 
\;\eqdef\; \pE \left[ \hpost{s\tbar t}{Z_s, \bx_t, \obs}{\bx_s}[\param]  \right] \eqsp,
\end{equation}
where $Z_s \sim \pdata{s\tbar t}{\bx_t}{}[\param]$, for \eqref{eq:posterior-transition}. The transition $\hpost{s\tbar t}{\bx_t, \obs}{\bx_s}[\param]$ generally lacks a closed-form expression; nevertheless, since it has a \emph{mixture structure}, it allows for straightforward and efficient sampling. Sampling from $\hpost{s\tbar t}{\bx_t, \obs}{\bx_s}[\param]$ can be performed by first drawing $Z_s$ from $\pdata{s\tbar t}{\bx_t}{}[\param]$, and then sampling from $\hpost{s\tbar t}{Z_s, \bx_t, \obs}{\bx_s}[\param]$. Moreover, as we will now show, in the case of inpainting, the second step can be carried out exactly.

Let $\bmu^\param_{s\tbar t}(\bx_t;\!\ddimstd)$ denote the mean of the Gaussian reverse transition $\pdata{s\tbar t}{\bx_t}{}[\ddimstd, \param]$. In the case of inpainting \eqref{eq:inpainting-likelihood}, standard Gaussian conjugacy results \cite[Equation~2.116]{bishop2006pattern} show that $\hpost{s\tbar t}{\bz_s, \bx_t, \obs}{}[\param]$ admits a closed-form Gaussian expression
\begin{multline}
\label{eq:gaussian-conjugacy}
\hpost{s\tbar t}{\tcz{\bz_s}, \bx_t, \obs}{\bx_s}[\param] 
= \normpdf
\big(\bx_s[\mask]; \bmu^\param_{s\tbar t}(\bx_t;\!\ddimstd)[\mask], \, \ddimstd^2_s \Id_{\dimx - \dimobs}\big) \\
\times \normpdf
\Big(\bx_s[\unmask]; (1 - \gamma_{s\tbar t}) \bmu^\param_{s\tbar t}(\bx_t;\!\ddimstd)[\unmask] 
+ \gamma_{s\tbar t}\big(\alpha_s \obs + \sigma_s \denoiser{s}{}{\tcz{\bz_s}}[\param][1][\unmask]\big), \, 
\alpha_s^2 \sigma_\obs^2 \gamma_{s\tbar t}\, \Id_{\dimobs}\Big) \eqsp,
\end{multline}
with
$
\gamma_{s\tbar t} \eqdef \eta_s^2 / (\eta_s^2 + \alpha_s^2 \sigma_\obs^2)
$. 
A derivation is provided in \Cref{sec:posterior-derivation}. 
Thus, a sample $X_s$ from \eqref{eq:algo-transition} can be drawn exactly by, first, generating a realization $\bz_s$ of $\smash{Z_s \sim \pdata{s\tbar t}{\bx_t}{}[\ddimstd, \param]}$ and, second, sampling $X_s[\mask]$ and $X_s[\unmask]$ conditionally independently from the two Gaussian distributions in \eqref{eq:gaussian-conjugacy}; see \Cref{algo:decoupled} for a pseudocode of this approach, which we refer to as {\algo} (see \Cref{sec:intro}).  

\paragraph{Practical implementation. }
A key practical feature of our method is that it depends on a single hyperparameter: the sequence $(\ddimstd_t)_{t \in [0,1]}$ of standard deviations, which controls the level of stochasticity in the DDIM reverse process. This choice is particularly critical in the low-NFE regime, where only a few function evaluations are available and the variance schedule strongly influences both observation fidelity and perceptual quality. In all experiments, we adopt the schedule $\ddimstd_t = \sigma_t (1 - \alpha_t)$.
An ablation study of this parameter is reported in \Cref{sec:experiments}.  

Beyond this hyperparameter, an important practical consideration is that most large-scale diffusion models for high-resolution image generation operate in a compressed latent space rather than in pixel space \citep{rombach2022high, esser2024scaling}. To apply our algorithm in this setting, we must therefore formulate the inpainting task in the latent space. Denote by $\encoder$ the encoder, $\mathbf{X}_\ast$ the pixel-space ground-truth image and $\textbf{M}$ the corresponding pixel-space mask. Following \citet{avrahami2022blended}, we set $\bx_\ast \eqdef \encoder(\mathbf{X}_\ast)$, the observation to $\obs \eqdef \bx_\ast[\mask]$ where $\mask$ is a downsampled version of the pixel-space mask $\mathbf{M}$. 
We illustrate in \Cref{fig:illustration-latent-masking} that masking in the latent spaces translates to masking in the pixel despite the nonlinearity of the decoder.
Since the encoder reduces spatial resolution by a fixed factor (\emph{e.g.}, $s=8$ in \cite{esser2024scaling}), we construct the latent mask $\mask$ by average pooling the binary pixel-space mask $\mathbf{M}$ with kernel and stride $s$. Each latent site is assigned the fraction of unmasked pixels within its receptive field. These fractional values are then thresholded (typically at $0.5$) to produce a binary mask; in other words, a latent site is marked as observed if the majority of its underlying pixels are unmasked. In practice, the mask $\mask$ is provided as a single-channel image and broadcast across all latent channels when applied to $\bx_\ast$. Finally, we apply \Cref{algo:decoupled} with $(\bx_\ast, \obs, \mask)$ thus defined in the latent space.

\paragraph{Related methods.} Our work shares similarities with various recent approaches to zero-shot diffusion guidance, which now briefly review.  
The closest line of work comprises variants of the \emph{replacement method} \citep{song2019generative, song2021score}, which follows the same structure as \Cref{algo:decoupled}. In these schemes, the masked coordinates of the state are updated according to the standard DDIM transition (Line~\ref{line:masked_update}), while the unmasked coordinates are replaced by a direct update that enforces consistency with the observation $\obs$ (Line~\ref{line:unmasked_update}).   In its simplest form, the method performs ancestral sampling with the transition
\begin{equation}
\label{eq:replacement-method}
\hpost{s\tbar t}{\bx_t, \obs}{\bx_s}[\param] 
= \normpdf
\big(\bx_s[\mask];\, \bmu^\param_{s\tbar t}(\bx_t;\!\ddimstd)[\mask],\, \eta_s^2 \Id_{\dimx - \dimobs}\big)\,
  \normpdf
  \big(\bx_s[\unmask];\, \alpha_s \obs,\, \sigma_s^2 \Id_{\dimobs}\big) \eqsp,
\end{equation}
\emph{i.e.}, the unmasked state is set to a noisy version of the observation $\alpha_s \obs + \sigma_s W_s$, where $W_s \sim \gauss(\zero, \Id_{\dimobs})$; see \citet[Algorithm~2]{song2019generative} and \citet[Appendix~I.2]{song2021score}. \cite{avrahami2022blended} extended this approach to the latent space using a downsampled mask. The method was later refined in RePaint \citep{lugmayr2022repaint}, which improves sample quality by performing multiple back-and-forth updates: after applying the replacement step from $\tk{k+1}$ to $\tk{k}$, a forward noising step is applied from $\tk{k}$ back to $\tk{k+1}$, and this cycle is repeated several times.  Several works have combined the replacement method with sequential Monte Carlo (SMC) sampling \citep{trippe2023diffusion, cardoso2023monte, dou2024diffusion, corenflos2025conditioning, zhao2025generative}. In particular, \citet{cardoso2023monte} update the unmasked coordinates of each particle using a Gaussian transition whose mean is a convex combination of the DDIM mean and the rescaled observation $\alpha_{\tk{k}} \obs$. In the inpainting framework, the recently proposed PnP-Flow \citep{martin2025pnpflow} reduces to using similar transitions without relying on SMC, \emph{i.e.}, by using a single particle. 
We explicitly compare the update rules in \citet{cardoso2023monte,martin2025pnpflow} with ours in \Cref{subsec:related-works}, where we also discuss additional related work. 

\section{Experiments}
\label{sec:experiments}
In this section, we extensively evaluate the inpainting performance of \algo\
when used with different large-scale models.
We benchmark its performance on multiple datasets against several state-of-the-art baselines.
We further analyze the relevance of our modeling choices, specifically the formulation of the approximation in \eqref{eq:approx-decoupled-potential} and the schedule of DDIM standard deviations $(\ddimstd_t)_{t \in [0,1]}$, through a series of targeted ablations.

\paragraph{Models and datasets.}
We evaluate our method on Stable Diffusion~3.5 (medium) \citep{esser2024scaling}. We set the CFG scale to $2$.
Our experiments cover three datasets: \ffhq~\citep{karras2019ffhq}, \div2k~\citep{agustsson2017ntire}, and \PieBench~\citep{ju2024pnp-inversion}.
For \ffhq, we use the first 5k images and condition generation on the prompt \emph{``a high-quality photo of a face''}.
For \div2k, we include both training and validation splits (900 images in total), and generate captions for each image using BLIP-2 \citep{li2023blip}; see \Cref{sec:exp-details} for details.
All \ffhq\ and \div2k images are resized to a resolution of $768 \times 768$.
The \PieBench\ dataset contains $700$ images of resolution $512 \times 512$, each paired with an inpainting mask and an edit caption. After removing cases where the mask completely covers the image, the resulting evaluation set contains $556$ images.

\begin{figure}[t]
    \centering
    \vspace{-4mm}
    \includegraphics[width=0.8\textwidth]{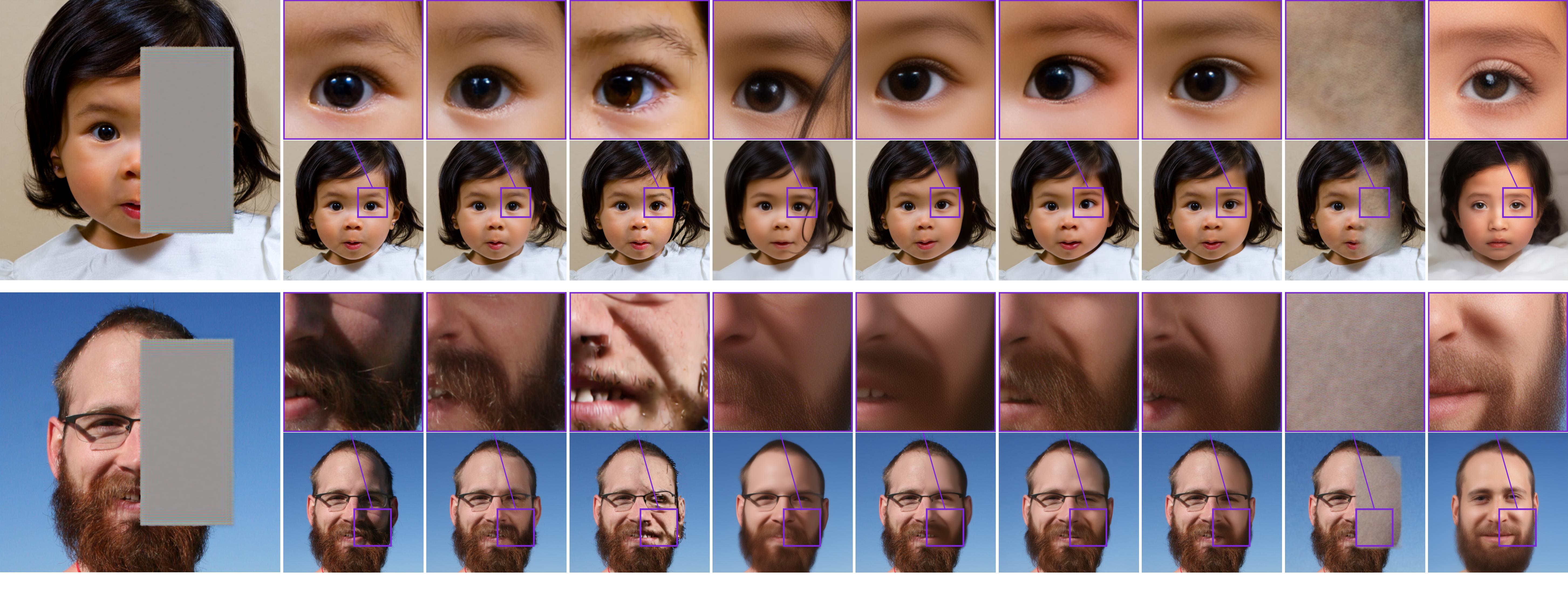}
    \includegraphics[width=0.8\textwidth]{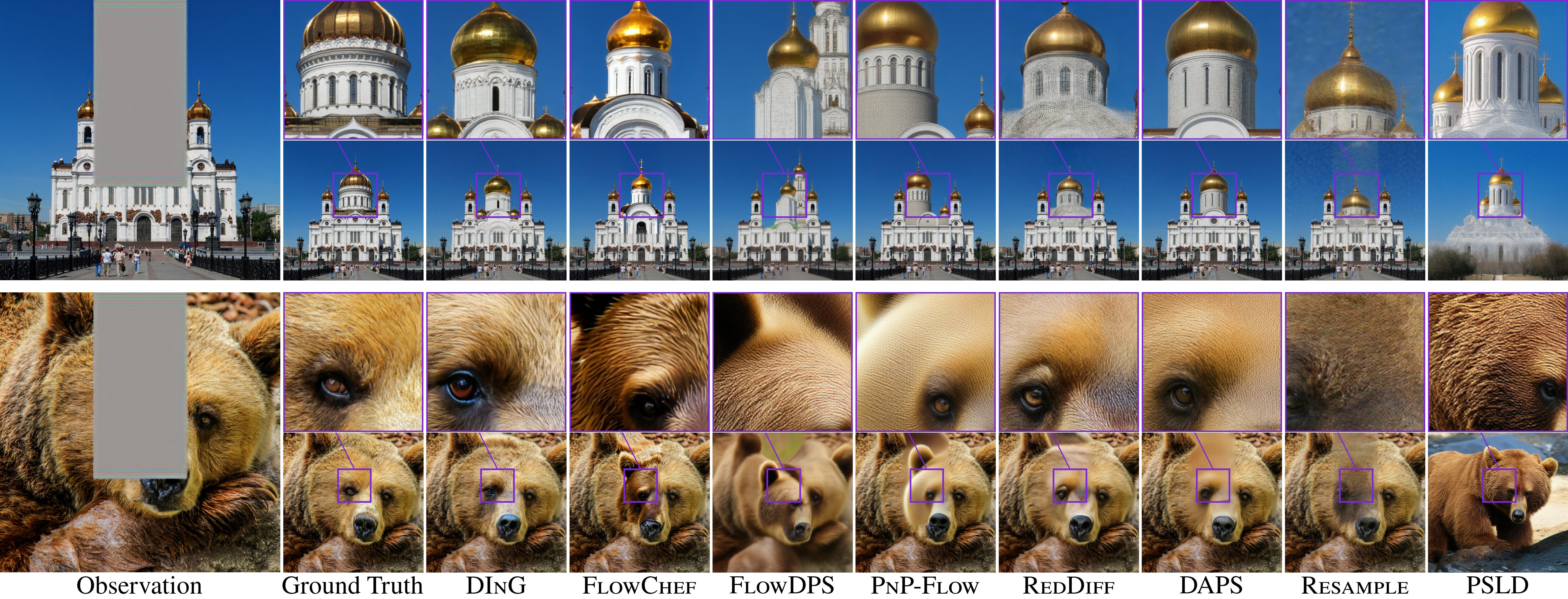}
    \captionsetup{font=small}
    \caption{Examples of reconstructions on \ffhq\ and \div2k with 50 NFEs.}
    \label{fig:DIV2K_samples}
    \vspace{-4mm}
\end{figure}

\paragraph{Evaluation and masks.}
For \ffhq\ and \div2k, we evaluate inpainting performance under four rectangular masking configurations: (i) right half of the image (\emph{Half}), (ii) upper half (\emph{Top}), (iii) lower half (\emph{Bottom}), and (iv) a central $512 \times 512$ square (\emph{Center}). In contrast, \PieBench\ provides irregular masks with diverse spatial patterns; see \Cref{sec:reconstructions} for examples. Unless otherwise stated, we set $\std_\obs = 0.01$ across all tasks. Since exact posterior sampling is infeasible, we assess inpainting quality using proxy metrics. To measure distributional alignment with the dataset, we report both FID and patch FID (pFID)~\citep{chai2022pfid}, the latter offering finer granularity for high-resolution evaluation. Following the standard FID protocol, we extract 10 random $256 \times 256$ patches per image, yielding a total of 50k patches. To quantify consistency with the observed content, we compute context PSNR (cPSNR), defined as the PSNR over the unmasked region only. We further report LPIPS~\citep{zhang2018lpips} relative to the ground truth to evaluate perceptual similarity, which is especially relevant for \ffhq\ where facial symmetries make reconstructions visually close to the reference. For \PieBench, which includes edit captions, we additionally report CLIP-Score~\citep{radford2021learning} on both the full image (CLIP) and the edited region (CLIP-ED), following \citet{ju2024pnp-inversion}. Together, these metrics provide a comprehensive evaluation of inpainting quality. While each captures a different aspect of performance, none should be interpreted in isolation.

\begin{wraptable}[11]{r}{0.3\textwidth}
    \centering
    \captionsetup{font=small}
    \vspace*{-4mm}
    \caption{Memory and runtime.}
    \vspace*{-2mm}
\resizebox{0.3\textwidth}{!}{
\begin{tabular}{lrr}
\toprule
& \multicolumn{2}{c}{\texttt{DIV2K} 768px} \\
\cmidrule(lr){2-3}
Method & Time (s) & Mem. (GB) \\
\midrule
 \blended  & 3.0 & 22.09 \\
\daps & 9.1 & 22.09 \\
 \ddnm & 3.1 & 22.09 \\
 \diffpir & 3.1 & 22.09 \\
\flowchef & 3.0 & 22.09 \\
\flowdps & 3.0 & 22.10 \\
\pnpflow & 3.1 & 22.09 \\
\psld & 7.4 & 24.49 \\
\reddiff & 3.1 & 22.09 \\
\resample & 8.1 & 24.50 \\
\rowcolor{deepcarrotorange!60} \algo\ (ours) & 2.9 & 22.09 \\
\bottomrule
\end{tabular}
    }
\label{tab:runtime}
\end{wraptable}
\paragraph{Baselines.}  We compare against seven state-of-the-art baselines: \flowchef~\citep{patel2024flowchef}, \flowdps~\citep{kim2025flowdps}, \daps~\citep{zhang2025improving}, \reddiff~\citep{mardani2024a}, \resample~\citep{song2024solving}, \psld~\citep{rout2024solving}, \pnpflow~\citep{martin2025pnpflow}, \rebuttal{\diffpir~\citep{zhu2023denoising}, \ddnm~\citep{wang2023zeroshot} and \blended~\citep{avrahami2023blendedlatent}}
For the main comparison, all methods are evaluated under a fixed budget of 50 NFEs. Since our method requires two denoiser evaluations per diffusion step, we use 25 steps to match this budget. We focus on this low-NFE regime as it reflects realistic settings, where inference is constrained by latency and compute. To ensure fairness, all methods are run in the latent space with downsampled masks, and extensive hyperparameter tuning is performed for each baseline on each dataset. For baselines that require VJP or backpropagation through the denoiser, we report their actual runtime and memory costs, ensuring that comparisons reflect effective inference cost rather than nominal NFE counts. Average runtime and memory usage across all the experiments, measured on H100 GPUs, are provided in \Cref{tab:runtime}.

\begin{table*}[t]
    \centering
    \captionsetup{font=small}
    \caption{\textbf{Top}: Quantitative results on \ffhq\ $768 \times 768$ with 5k samples. \textbf{Bottom}: \div2k $768 \times 768$ with $900$ samples. For FID, pFID, and LPIPS, the lower the better. For cPSNR, the higher the better. \rebuttal{50 NFEs were used}.}
    \vspace{-.3cm}
\resizebox{\textwidth}{!}{
\begin{tabular}{lcccc|cccc|cccc|cccc}
\toprule
 & \multicolumn{4}{c}{Half} & \multicolumn{4}{c}{Center} & \multicolumn{4}{c}{Top} & \multicolumn{4}{c}{Bottom} \\
\cmidrule(lr){2-5} \cmidrule(lr){6-9} \cmidrule(lr){10-13} \cmidrule(lr){14-17}
Method & FID & pFID & cPSNR & LPIPS & FID & pFID & cPSNR & LPIPS & FID & pFID & cPSNR & LPIPS & FID & pFID & cPSNR & LPIPS \\
\midrule
& \multicolumn{16}{c}{\textbf{FFHQ} \; $768 \times 768$} \\
\midrule

 \blended
& 23.5 & 16.3 & \textbf{31.32} & 0.38
& 35.3 & 36.7 & 31.54 & 0.33
& 32.8 & 15.8 & \underline{32.05} & 0.38
& 43.7 & 19.8 & \textbf{30.85} & 0.37 \\

\daps
& 17.9 & 25.1 & 30.50 & {0.36}
& 35.1 & 54.5 & 31.15 & 0.32
& 30.1 & 30.5 & 31.54 & 0.39
& 52.8 & 27.6 & 30.30 & 0.34 \\

 \ddnm
& 12.3 & 13.68 & \underline{31.27} & \textbf{0.33}
& 24.4 & 34.8 & {31.61} & \textbf{0.27}
& {22.3} & 23.2 & 31.82 & \underline{0.36}
& 38.3 & 19.6 & 30.51 & \textbf{0.32} \\

 \diffpir
& \underline{12.1} & \underline{11.23} & 30.91 & {0.36}
& \underline{19.4} & \underline{19.6} & \textbf{31.67} & {0.30}
& \textbf{19.7} & \underline{14.1} & \textbf{32.07} & \underline{0.36}
& \underline{30.7} & \underline{11.4} & \underline{30.74} & 0.35 \\

\flowchef
& 20.2 & 16.5 & 30.41 & {0.36}
& 29.3 & 35.0 & 31.00 & 0.31
& 27.8 & 21.1 & 31.05 & \underline{0.36}
& 35.9 & 22.7 & 29.94 & 0.35 \\

\flowdps
& 36.2 & 49.0 & 26.72 & 0.46
& 49.5 & 79.9 & 23.36 & 0.53
& 52.1 & 74.2 & 24.15 & 0.56
& 72.3 & 71.5 & 23.06 & 0.55 \\

\pnpflow
& 20.5 & 33.4 & 30.62 & 0.37
& 36.6 & 65.1 & \textbf{31.67} & 0.32
& 33.6 & 42.7 & 31.54 & 0.38
& 56.8 & 33.4 & 29.95 & {0.33} \\

\psld
& 116.3 & 73.8 & 6.89 & 0.81
& 98.0  & 69.1 & 10.09 & 0.73
& 120.6 & 75.4 & 7.06 & 0.81
& 107.0 & 70.4 & 6.46 & 0.81 \\

\reddiff
& 28.5 & 37.9 & 27.39 & 0.39
& 30.7 & 41.8 & 27.85 & 0.32
& 33.0 & 41.1 & 27.92 & 0.41
& 76.4 & 41.3 & 26.96 & 0.39 \\

\resample
& 32.4 & 48.8 & 28.53 & 0.44
& 53.8 & 103.4 & 28.46 & 0.40
& 63.2 & 56.2 & 29.02 & 0.44
& 97.8 & 57.0 & 28.06 & 0.44 \\

\rowcolor{deepcarrotorange!60} \algo\ (ours)
& \textbf{9.6} & \textbf{6.6} & 31.03 & \textbf{0.33}
& \textbf{15.5} & \textbf{14.0} & 31.38 & \textbf{0.27}
& \textbf{19.7} & \textbf{12.5} & 31.64 & \textbf{0.34}
& \textbf{29.6} & \textbf{8.6} & 30.50 & \textbf{0.32} \\
\midrule

& \multicolumn{16}{c}{\textbf{DIV2K} \; $768 \times 768$} \\
\midrule

 \blended
& 43.6 & \underline{12.9} & \underline{26.03} & {0.37}
& 54.8 & \underline{20.2} & 26.43 & 0.35
& 44.8 & \underline{13.2} & {25.28} & {0.39}
& 48.1 & \textbf{13.1} & \underline{26.85} & 0.38 \\

\daps
& 51.0 & 38.4 & 25.92 & 0.46
& 74.8 & 67.6 & 26.14 & 0.44
& 54.8 & 41.0 & 25.22 & 0.44
& 61.2 & 39.7 & 26.71 & 0.50 \\

 \ddnm
& 42.5 & 21.2 & \underline{26.03} & 0.41
& 57.7 & 38.5 & \underline{26.61} & 0.37
& 45.7 & 21.3 & \textbf{25.36} & 0.42
& 49.6 & 23.2 & 26.81 & 0.45 \\

 \diffpir
& \underline{41.1} & \underline{12.9} & \textbf{26.09} & {0.37}
& \underline{52.8} & 21.4 & 26.58 & 0.34
& \underline{43.5} & \textbf{13.1} & \textbf{25.36} & {0.39}
& \underline{44.9} & 14.9 & \textbf{26.91} & 0.39 \\

\flowchef
& 43.3 & \textbf{12.2} & 25.78 & \underline{0.36}
& {53.6} & 22.3 & 26.27 & \underline{0.32}
& 45.0 & 13.8 & 25.09 & \textbf{0.37}
& 46.9 & \underline{13.2} & 26.57 & \textbf{0.37} \\

\flowdps
& 50.8 & 33.2 & 21.30 & 0.49
& 70.3 & 62.8 & 18.38 & 0.63
& 64.1 & 57.9 & 17.43 & 0.65
& 64.2 & 57.3 & 19.06 & 0.63 \\

\pnpflow
& 54.2 & 42.7 & 26.00 & 0.46
& 79.7 & 71.4 & \textbf{26.63} & 0.44
& 57.1 & 33.5 & 25.19 & 0.44
& 64.7 & 50.8 & 26.61 & 0.50 \\

\psld
& 66.4 & 32.3 & 6.15 & 0.79
& 66.9 & 35.7 & 9.89 & 0.72
& 66.5 & 31.9 & 6.35 & 0.79
& 66.3 & 32.6 & 6.41 & 0.78 \\

\reddiff
& 54.2 & 45.7 & 22.64 & 0.49
& 77.4 & 69.8 & 23.25 & 0.46
& 57.7 & 40.9 & 22.17 & 0.48
& 60.6 & 46.3 & 23.41 & 0.52 \\

\resample
& 52.7 & 34.1 & 23.33 & 0.47
& 80.8 & 63.8 & 23.76 & 0.43
& 56.1 & 33.1 & 22.84 & 0.47
& 60.6 & 41.0 & 24.06 & 0.48 \\

\rowcolor{deepcarrotorange!60} \algo\ (ours)
& \textbf{39.2} & 13.0 & 25.90 & \textbf{0.35}
& \textbf{50.7} & \textbf{19.5} & 26.41 & \textbf{0.31}
& \textbf{41.4} & 13.7 & 25.19 & \textbf{0.37}
& \textbf{43.4} & 13.4 & 26.72 & \textbf{0.37} \\
\bottomrule

\end{tabular}
\label{tab:ffhqdiv2k}
}
\vspace{-.4cm}
\end{table*}

\subsection{Main results}
\begin{wraptable}{r}{0.4\textwidth}
    \centering
    \captionsetup{font=small}
    \caption{Results on \PieBench\ with $556$ samples and 50 NFEs.}
    \resizebox{0.4\textwidth}{!}{
\begin{tabular}{lcccccc}
\toprule
Method & FID & pFID & cPSNR & LPIPS & CLIP & CLIP-ED \\
\midrule

 
\blended 
& 65.5 
& 27.0 
& 26.60 
& {0.31} 
& \underline{26.32} 
& \underline{23.15} \\

\daps 
& 65.9 
& 30.2 
& \underline{27.08} 
& 0.34 
& 25.57 
& 21.75 \\

 
\ddnm 
& \textbf{61.4} 
& 26.9 
& \textbf{27.29} 
& {0.31} 
& 26.27 
& 22.96 \\

 
\diffpir 
& \underline{63.5} 
& \underline{25.4} 
& 26.98 
& \textbf{0.30} 
& 26.21 
& 23.04 \\

\flowchef 
& 68.3 
& 27.4 
& 26.84 
& \textbf{0.30} 
& 26.02 
& 22.47 \\

\flowdps 
& 74.6 
& 42.7 
& 22.05 
& 0.45 
& \textbf{26.35} 
& 22.79 \\

\pnpflow 
& 66.8 
& 32.1 
& 26.90 
& 0.34 
& 25.62 
& 21.02 \\

\psld 
& 94.1 
& 56.2 
& 14.25 
& 0.65 
& \underline{26.32} 
& 21.81 \\

\reddiff 
& 69.5 
& 35.2 
& 24.34 
& 0.37 
& 25.27 
& 21.18 \\

\resample 
& 71.0 
& 33.9 
& 24.45 
& 0.35 
& 25.71 
& 22.03 \\

\rowcolor{deepcarrotorange!60} 
\algo\ (ours) 
& \textbf{61.4} 
& \textbf{24.7} 
& 27.03 
& \textbf{0.30} 
& 26.30 
& \textbf{23.36} \\

\bottomrule
\end{tabular}
    }
\label{tab:PIE}
\end{wraptable}

Tables \ref{tab:ffhqdiv2k} and \ref{tab:PIE} summarize the results on \ffhq, \div2k\ and \PieBench, respectively. On FFHQ (Table 3), DING achieves the best performance on all masks and almost all the metrics. In particular, it improves both pFID and FID by significant margins over the strongest competing method \flowchef. It also obtains the highest cPSNR scores, indicating a faithful reconstruction of the visible content, while simultaneously producing visually coherent completions with the lowest LPIPS. On \div2k, the comparison is more nuanced. \algo\ consistently attains the best FID and LPIPS across all four masks, while remaining competitive on pFID and comparable to strong baselines on cPSNR. On \PieBench, \algo\ achieves the best results on all metrics except cPSNR and CLIP. We note, however, that although \flowdps\ and \psld\ obtain slightly higher CLIP scores, they perform markedly worse on fidelity and perceptual quality metrics, suggesting that their improvements in CLIP may reflect metric hacking rather than genuine reconstruction quality. We provide qualitative comparisons of the reconstruction in \Cref{fig:sd3ft} and \Cref{sec:reconstructions}. 

We now compare \algo\ with a Stable Diffusion 3 model fine-tuned for inpainting\footnote{\url{https://huggingface.co/alimama-creative/SD3-Controlnet-Inpainting}}, trained on 12M images at $1024\times1024$ resolution. To ensure fairness, both models are evaluated under the same runtime budget (\textbf{2.2s}), which corresponds to 56 NFEs for \algo\ and 28 NFEs for the fine-tuned baseline. \rebuttal{We also provide the results for the finetuned model using 56 NFEs.}
\begin{wraptable}{r}{0.5\textwidth}
    \centering
    \captionsetup{font=small}
    \caption{Results on the \PieBench\ with $556$ samples.}
    \vspace{-.3cm}
    \resizebox{0.5\textwidth}{!}{
\begin{tabular}{lcccccc}
\toprule
& FID & pFID & cPSNR & LPIPS & CLIP & CLIP-ED \\
\textcolor{red}{3 seconds} & \\
\midrule
SD3 Inpaint (28) & 68.7 & 30.5 & 18.85 & \underline{0.34} & 26.37 & \underline{23.10} \\ 
\rowcolor{deepcarrotorange!60} \algo\ (ours) & \textbf{63.6} & \textbf{24.6} & \textbf{26.98} & \textbf{0.30} & \textbf{26.63} & \textbf{23.70} \\    
\bottomrule
\end{tabular}
    }
\label{tab:PIE-ft}
\end{wraptable}
The results are given in Tables~\ref{tab:finetuning} and \ref{tab:PIE-ft}. Across FFHQ, DIV2K, and \PieBench, \algo\ consistently outperforms the fine-tuned SD3 model on all reported metrics. The gains are especially pronounced in cPSNR, where \algo\ achieves 8–10 dB higher fidelity to the observed pixels. This indicates that our method preserves the known content far more accurately while still producing realistic completions, as confirmed by lower FID and LPIPS. On \PieBench, \algo\ further improves over the fine-tuned baseline on every metric, including perceptual ones (pFID, LPIPS), while also yielding stronger text–image alignment (CLIP, CLIP-ED). See \Cref{fig:sd3ft} for a qualitative comparison of the reconstructions. These results demonstrate that, even without task-specific fine-tuning on a large amount of images, \algo\ not only matches but surpasses a specialized SD3 inpaint model.
Overall, these results show that our method provides the strongest overall trade-off between realism and fidelity under low NFE budgets. 
\begin{table}[H]
    \centering
    \captionsetup{font=small}
    \caption{\algo\ compared to SD3 fine-tuned (28 and 56 NFEs) for inpainting tasks.}
    \vspace{-.3cm}
\resizebox{\textwidth}{!}{
\begin{tabular}{lcccc|cccc|cccc|cccc}
\toprule
 & \multicolumn{4}{c}{Half} & \multicolumn{4}{c}{Center} & \multicolumn{4}{c}{Top} & \multicolumn{4}{c}{Bottom} \\
\cmidrule(lr){2-5} \cmidrule(lr){6-9} \cmidrule(lr){10-13} \cmidrule(lr){14-17}
Method & FID & pFID & cPSNR & LPIPS & FID & pFID & cPSNR & LPIPS & FID & pFID & cPSNR & LPIPS & FID & pFID & cPSNR & LPIPS \\
\midrule
& \multicolumn{16}{c}{\textbf{FFHQ} $512 \times 512$} \\
\midrule
SD3 Inpaint (28) & \underline{23.5} & 10.7 & \underline{21.69} & \underline{0.37} & \underline{62.1} & \underline{33.9} & \underline{22.18} & \underline{0.31} & \underline{34.7} & 17.8 & \underline{21.64} & \underline{0.36} & \underline{42.4} & \underline{16.5} & \underline{21.78} & 0.37 \\

 SD3 Inpaint (56) & 23.7 & \underline{10.3} & 21.53 & \underline{0.37} & 63.7 & 34.4 & 21.94 & \underline{0.31} & 35.4 & \underline{16.5} & 21.41 & \underline{0.36} & 43.8 & 16.8 & 21.53 & \underline{0.36} \\

\rowcolor{deepcarrotorange!60} \algo\ (ours) & \textbf{9.3} & \textbf{5.8} & \textbf{31.40} & \textbf{0.32} & \textbf{20.2} & \textbf{15.5} & \textbf{31.39} & \textbf{0.28} & \textbf{17.3} & \textbf{8.4} & \textbf{31.96} & \textbf{0.33} & \textbf{33.8} & \textbf{12.2} & \textbf{31.27} & \textbf{0.34} \\
\midrule
& \multicolumn{16}{c}{\textbf{DIV2K} $512 \times 512$} \\
\midrule
SD3 Inpaint (28) & 45.9 & 15.0 & 17.95 & \underline{0.40} & \underline{54.2} & 22.1 & 18.57 & \underline{0.36} & 48.8 & 16.6 & 18.12 & 0.42 & 51.0 & 17.5 & 18.95 & \underline{0.41} \\

 SD3 Inpaint (56) & \underline{45.1} & \textbf{14.0} & \underline{17.91} & \underline{0.40} & \underline{54.2} & \textbf{20.5} & \underline{18.63} & \underline{0.36} & \underline{48.6} & \underline{16.1} & \underline{18.16} & \underline{0.41} & \underline{50.3} & \underline{17.2} & \underline{19.02} & \underline{0.41} \\

\rowcolor{deepcarrotorange!60} \algo\ (ours) & \textbf{41.5} & \underline{14.2} & \textbf{26.09} & \textbf{0.37} & \textbf{52.4} & \underline{21.5} & \textbf{26.53} & \textbf{0.33} & \textbf{43.7} & \textbf{13.8} & \textbf{25.47} & \textbf{0.38} & \textbf{45.4} & \textbf{15.4} & \textbf{26.94} & \textbf{0.38} \\
\bottomrule
\end{tabular}
\label{tab:finetuning}
}
\end{table}
\subsection{Ablations}
\label{subsec:ablation}
\paragraph{Doubled NFE per diffusion step.}
Because the $\bx_1$-predictor must be evaluated at the proxy variable (Line~\ref{line:noise_proxy} in \Cref{algo:decoupled}), our algorithm requires two NFEs per diffusion step.
An immediate question is whether this overhead is needed.
To explore this, we introduce a variant in which the noise prediction from the previous step is reused instead of being recomputed at the proxy. We coin this variant as Delayed \algo, where Line~\ref{line:noise_proxy} is replaced by
$
    \hat\bx^{\mathrm{pxy}}_1 \gets (\textcolor{ruddy}{\bx} - \sigma_\tk{k} \noisepred{\tk{k+1}}{}{\textcolor{ruddy}{\bx}}) / \acp_\tk{k},
$
and we further set $\ddimstd_t = \std_t (1 - \acp_t)$, which we found to yield the best performance in this setting.
A quantitative comparison with the original \algo\ is reported in \Cref{tab:delayed_ablation}, showing that while Delayed \algo\ reduces the NFE cost per step, it consistently underperforms across metrics and masking patterns, indicating that the doubled NFE is necessary to retain the full effectiveness of our approach.
\begin{table}[H]
    \centering
    \captionsetup{font=small}
    \caption{Delayed \algo\ compared to \algo\ on FFHQ (5k samples) and DIV2K (900 samples) with 50 NFEs.}
    \vspace{-.3cm}
\resizebox{\textwidth}{!}{
\begin{tabular}{lcccc|cccc|cccc|cccc}
\toprule
Method & \multicolumn{4}{c}{Half} & \multicolumn{4}{c}{Center} & \multicolumn{4}{c}{Top} & \multicolumn{4}{c}{Bottom} \\
\cmidrule(lr){2-5} \cmidrule(lr){6-9} \cmidrule(lr){10-13} \cmidrule(lr){14-17}
 & pFID & FID & cPSNR & LPIPS & pFID & FID & cPSNR & LPIPS & pFID & FID & cPSNR & LPIPS & pFID & FID & cPSNR & LPIPS \\
\midrule
& \multicolumn{16}{c}{\textbf{FFHQ} $768 \times 768$} \\
\midrule
Delayed \algo & \textbf{9.1} & \underline{7.4} & \underline{29.21} & \textbf{0.33} & \underline{21.3} & \underline{20.7} & \underline{29.90} & \textbf{0.26} & \textbf{15.7} & \textbf{9.9} & \underline{29.84} & \textbf{0.33} & \underline{31.0} & \underline{12.3} & \underline{28.88} & \underline{0.33} \\
\rowcolor{deepcarrotorange!60} \algo  & \underline{9.6} & \textbf{6.6} & \textbf{31.03} & \textbf{0.33} & \textbf{15.5} & \textbf{14.0} & \textbf{31.38} & \underline{0.27} & \underline{19.7} & \underline{12.5} & \textbf{31.64} & \underline{0.34} & \textbf{29.6} & \textbf{8.6} & \textbf{30.50} & \textbf{0.32} \\
\midrule
& \multicolumn{16}{c}{\textbf{DIV2K} $768 \times 768$} \\
\midrule
Delayed \algo & \underline{43.9} & \underline{15.9} & \underline{24.88} & \underline{0.36} & \underline{55.6} & \underline{24.7} & \underline{25.52} & \underline{0.32} & \underline{45.3} & \underline{16.8} & \underline{24.36} & \underline{0.38} & \underline{47.8} & \underline{14.6} & \underline{25.64} & \underline{0.38} \\
\rowcolor{deepcarrotorange!60} \algo & \textbf{39.2} & \textbf{13.0} & \textbf{25.90} & \textbf{0.35} & \textbf{50.7} & \textbf{19.5} & \textbf{26.41} & \textbf{0.31} & \textbf{41.4} & \textbf{13.7} & \textbf{25.19} & \textbf{0.37} & \textbf{43.4} & \textbf{13.4} & \textbf{26.72} & \textbf{0.37} \\
\bottomrule
\end{tabular}
\label{tab:delayed_ablation}
}
\end{table}
\paragraph{DDIM schedule.}
Here we proceed to compare the behavior of our algorithm under different schedules $(\ddimstd_t)_{t \in [0, 1]}$. For this purpose we compare against some natural candidates. \tcp{(A)}: we consider the DDPM schedule used in \citet{ho2020denoising} and which corresponds to using in \eqref{eq:ddpm-reverse} a standard deviation that depends on both $s$ and $t$, \emph{i.e.}, $\smash{\ddimstd_s(t) = \std_s (\std^2 _t - (\acp_t / \acp_s)^2 \std^2 _s)^{1/2} / \std _t}$. \tcp{(B)}: as we cannot use deterministic sampling in our approach, we rescale the DDPM schedule (A) with $0.01$ to approach deterministic sampling. \tcp{(C)}: $\ddimstd_s = \std_s$, which is the maximum allowed standard deviation in \eqref{eq:ddpm-reverse}. In this scenario, the transition is $\pdata{s\tbar t}{\bx_t}{\bx_s}[\param] = \normpdf(\bx_s; \denoiser{t}{}{\bx_t}[\param], \std^2 _s \Id_\dimx)$ and resembles the prior transition used in \cite{martin2025pnpflow}. \tcp{(D)}: $\ddimstd_s = \std_s \sqrt{1 - \acp_s}$, which corresponds to a slower decay of the standard deviation compared to our default choice $\ddimstd_s = \std_s (1 - \acp_s)$. The results on FFHQ with 5k samples and 50 NFEs are reported in \Cref{tab:schedule_ablation}. 
We observe that the rescaled DDPM schedule (B) degrades significantly across all metrics, while (A) and (C) yield nearly identical performance. This suggests that maintaining sufficient stochasticity at the beginning of the diffusion process is crucial for strong performance.
Among the alternatives, (D) performs best; still, it is outperformed by our default schedule, confirming the benefit of a faster decay of $(\ddimstd_t)$. 

\begin{table}[H]
    \centering
    \captionsetup{font=small}
    \caption{Ablation results for the DDIM schedule $(\ddimstd_t)$ on \ffhq\ $768 \times 768$ with 5k samples and 50 NFEs.}
    \vspace{-.3cm}
\resizebox{\textwidth}{!}{
\begin{tabular}{lcccc|cccc|cccc|cccc}
\toprule
Method & \multicolumn{4}{c}{Half} & \multicolumn{4}{c}{Center} & \multicolumn{4}{c}{Top} & \multicolumn{4}{c}{Bottom} \\
\cmidrule(lr){2-5} \cmidrule(lr){6-9} \cmidrule(lr){10-13} \cmidrule(lr){14-17}
 & FID & pFID & cPSNR & LPIPS & FID & pFID & cPSNR & LPIPS & FID & pFID & cPSNR & LPIPS & FID & pFID & cPSNR & LPIPS \\
\midrule
(A) & 13.9 & 14.0 & \underline{31.19} & 0.36 & 19.1 & 25.5 & \underline{31.50} & 0.30 & 21.8 & 18.9 & \underline{31.80} & 0.38 & 35.1 & 15.1 & \underline{30.70} & 0.35 \\
(B) & 21.5 & 18.7 & 26.06 & 0.41 & 29.0 & 31.9 & 26.23 & 0.35 & 31.7 & 21.1 & 26.64 & 0.41 & 48.4 & 28.6 & 25.56 & 0.40 \\
(C) & 13.9 & 14.2 & \underline{31.19} & 0.36 & 19.1 & 25.5 & \underline{31.50} & 0.30 & 21.8 & 19.0 & \underline{31.80} & 0.38 & 35.1 & 15.6 & \underline{30.70} & 0.35 \\
(D) & \underline{10.2} & \underline{10.7} & \textbf{31.33} & \underline{0.33} & \underline{16.7} & \underline{19.0} & \textbf{31.70} & \underline{0.27} & \textbf{19.6} & \underline{15.7} & \textbf{31.95} & \underline{0.35} & \underline{31.6} & \underline{12.0} & \textbf{30.81} & \underline{0.32} \\
\rowcolor{deepcarrotorange!60} Default & \textbf{9.6} & \textbf{6.6} & 31.03 & \textbf{0.33} & \textbf{15.5} & \textbf{14.0} & 31.38 & \textbf{0.27} & \underline{19.7} & \textbf{12.5} & 31.64 & \textbf{0.34} & \textbf{29.6} & \textbf{8.6} & 30.50 & \textbf{0.32} \\
\bottomrule
\end{tabular}
}
\label{tab:schedule_ablation}
\end{table}

\section{Conclusion}
We have introduced {\algo}, a novel diffusion-based method for zero-shot inpainting that operates fully in the latent space and enables fast, memory-efficient inference under low-NFE budgets. Through extensive experiments across multiple benchmarks, we have shown that {\algo} consistently outperforms existing zero-shot approaches and even surpasses a fine-tuned Stable Diffusion 3 model for image editing, despite requiring no expensive training. Notably, our method produces globally coherent reconstructions while preserving the visible content with high fidelity.

\paragraph{Limitations and future directions.} While these results highlight the effectiveness and practicality of \algo, several avenues remain open. An important limitation of our current approach is that performance does not monotonically improve as the compute budget increases. Ideally, one would like reconstruction accuracy to keep improving with additional sampling steps, potentially beyond the standard diffusion horizon, but we observe diminishing returns due to the limitations of our current DDIM schedule. Addressing this issue, for example by designing guidance schemes or noise schedules that continue to scale gracefully with compute, remains an important direction for future work. Moreover, while our framework is fully operational in the latent space, its applicability is currently limited to inpainting, as this is the only observation operator we can reliably lift to the latent space. Extending the method to accommodate more general forward operators and a broader class of inverse problems, while preserving the same level of efficiency achieved for inpainting, is a challenging yet promising direction for future research.

\paragraph{Reproducibility statement.} We place strong emphasis on reproducibility. To this end, we provide the full source code of our method along with implementations of all baseline methods used in the paper. Our repository also includes scripts to reproduce every experiment, as well as configuration files specifying all hyperparameters and settings for each baseline and experimental setup. Together, these resources ensure that all results reported in this work can be fully reproduced and easily extended.

\paragraph{Ethics statement.} While the proposed approach demonstrates clear benefits for applications in restoration, accessibility, and creative media, it also lies at the borderline of ethical considerations. Diffusion-based inpainting methods can be misappropriated for producing deceptive or harmful content, such as manipulated images or synthetic media that obscure authenticity. This dual-use nature highlights the need for proactive safeguards, including transparent usage guidelines, traceable model outputs, and continued development of forensic detection tools to ensure responsible integration of such technologies.

\section*{Acknowledgements}
The work of Badr Moufad has been supported by Technology Innovation Institute (TII), project Fed2Learn. The work of Eric Moulines has been partly funded by the European Union (ERC-2022-SYG-OCEAN-101071601). Views and opinions expressed are however those
of the author(s) only and do not necessarily reflect those of the European Union or the European Research Council Executive Agency. Neither the European Union nor the granting authority can be held responsible for them. This work was granted access to the HPC resources of IDRIS under the allocations 2025-AD011015980 and 2025-AD011016484 made by GENCI.

\newpage
\bibliographystyle{iclr2026_conference}
\bibliography{bibliography}

\clearpage
\newpage
\appendix
\begin{figure}[t]
    \centering
    \includegraphics[width=1.\textwidth]{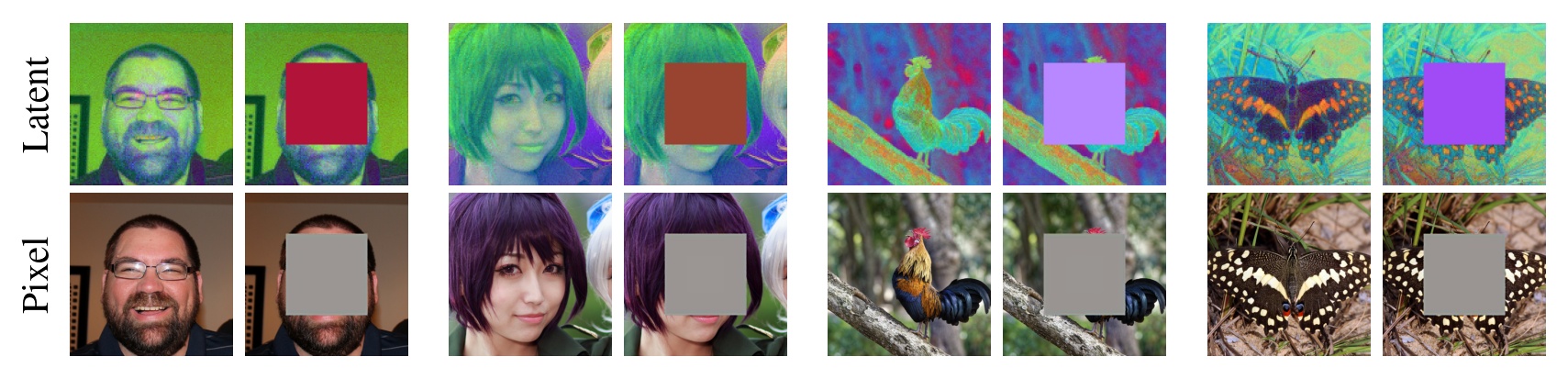}
    \captionsetup{font=small}
    \caption{%
    \rebuttal{Latent-space masking and its correspondence to pixel space using a central square mask.
    The encoder and decoder of Stable Diffusion 3.5 (medium) were used.
    The first row shows latent images alongside the encoded mask applied to each, while the second row shows their decoded counterparts.
    Notice that the masked regions in the latent space translate directly to analogous masked regions in pixel space.
    For that sake of visualization, since the latent images have 16 channels, we apply PCA and visualize the first 3 components.%
    }}
    \label{fig:illustration-latent-masking}
\end{figure}

\section{Method details}
\subsection{Derivation of the posterior \eqref{eq:posterior-transition}}
\label{sec:posterior-derivation}
Recall that given $\bz_s$, the posterior transition of interest is
\begin{align*}
\hpost{s\tbar t}{\bz_s, \bx_t, \obs}{\bx_s}[\param] 
& \propto \hpot{s}{\bx_s, \bz_s}{\obs}[\param] \; \pdata{s\tbar t}{\bx_t}{\bx_s}[\eta,\param] \eqsp.
\end{align*}
Denoting by $\tilde\obs_s = \acp_s \obs + \std_s \noisepred{s}{}{\bz_s}[\param][\unmask]$ the effective observation, we have that 
$$
\hpot{s}{\bx_s, \bz_s}{\obs}[\param] \propto \normpdf(\tcz{\tilde\obs_s}; \bx_s[\unmask], \acp^2 _s\std^2 _\obs \Id_\dimobs) \eqsp,
$$
and since the reverse transition writes 
$$
\pdata{s\tbar t}{\bx_t}{\bx_s}[\eta, \param] = \normpdf(\bx_s[\mask]; \mu^\param _{s\tbar t}(\bx_t;\!\ddimstd)[\mask], \ddimstd^2 _s \Id_{\dimx - \dimobs}) \normpdf(\bx_s[\unmask]; \mu^\param _{s\tbar t}(\bx_t;\!\ddimstd)[\unmask], \ddimstd^2 _s \Id_{\dimobs})
 \eqsp,$$ we obtain 
\begin{multline*} 
    \hpost{s\tbar t}{\bz_s, \bx_t, \obs}{\bx_s}[\param] = \normpdf(\bx_s[\mask]; \mu^\param _{s\tbar t}(\bx_t;\!\ddimstd)[\mask], \ddimstd^2 _s \Id_{\dimx - \dimobs}) \\
   \times \frac{\normpdf(\tilde\obs_s; \bx_s[\unmask], \acp^2 _s\std^2 _\obs \Id_\dimobs) \normpdf(\bx_s[\unmask]; \mu^\param _{s\tbar t}(\bx_t;\!\ddimstd)[\unmask], \ddimstd^2 _s \Id_{\dimobs})}{\int \normpdf(\tilde\obs_s; \tilde\bx_s[\unmask], \acp^2 _s\std^2 _\obs \Id_\dimobs) \normpdf(\tilde\bx_s[\unmask]; \mu^\param _{s\tbar t}(\bx_t;\!\ddimstd)[\unmask], \ddimstd^2 _s \Id_{\dimobs}) \, \rmd \tilde\bx_s[\unmask]} \eqsp.
\end{multline*}
The formula \eqref{eq:gaussian-conjugacy} follows by applying \citet[equation~2.116]{bishop2006pattern} to the second normalized transition on the right-hand side.  
\subsection{Comparison with related works}
We start by providing an explicit comparison with the closest works. 
\label{subsec:related-works}
\paragraph{Comparison with the transition in \cite{cardoso2023monte}.} Let $\tau \in [0, 1]$ be a timestep such that $\std _\obs = \sigma_\tau / \acp_\tau$. Such a $\tau$ always exists when the linear schedule is used for example. The transition used in the SMC algorithm in \cite{cardoso2023monte} for $s > \tau$ is given by
\rebuttal{
\begin{equation}
    \label{eq:mcgdiff-transition1}
    \hpost{s\tbar t}{\bx_t}{\bx_s}[\param] \propto \normpdf(\acp_s \obs; \bx_s[\unmask], \sigma^2 _{s\tbar \tau} \Id_\dimobs) \pdata{s\tbar t}{\bx_t}{\bx_s}[\ddimstd, \param]
\end{equation}
Using the same conjugation formulas as in the previous section, we find that}
\begin{multline} 
     \label{eq:mcgdiff-transition}
    \hpost{s\tbar t}{\bx_t}{\bx_s}[\param] = \normpdf(\bx_s[\mask], \mu^\param _{s\tbar t}(\bx_t;\!\ddimstd)[\mask], \ddimstd^2 _s \Id_{\dimx - \dimobs}) \\
     \times \normpdf(\bx_s[\unmask], (1 - \scdhilight{\tilde\gamma_{s\tbar t}})\bmu^\param _{s\tbar t}(\bx_t;\!\ddimstd)[\unmask] + \scdhilight{\tilde\gamma_{s\tbar t}} \hilight{\acp_s \obs},  \textcolor{purple}{\std^2 _{t\tbar \tau} \tilde\gamma_{s\tbar t}} \Id_{\dimx - \dimobs}),
\end{multline}
where $\std^2 _{t\tbar \tau} \eqdef \std^2 _t - (\acp_t / \acp_\tau)^2 \sigma^2 _\tau$ and $\tilde\gamma_{s\tbar t} = \ddimstd^2 _s / (\ddimstd^2 _s + \std^2 _{t\tbar \tau})$. This is to be contrasted with our update, given a sample $\bz_s$, 
\begin{multline*}    
\hpost{s\tbar t}{\bz_s, \bx_t, \obs}{\bx_s}[\param] 
= \normpdf
\big(\bx_s[\mask]; \bmu^\param_{s\tbar t}(\bx_t;\!\ddimstd)[\mask], \, \ddimstd^2_s \Id_{\dimx - \dimobs}\big) \\
\times \normpdf 
\Big(\bx_s[\unmask]; (1 - \scdhilight{\gamma_{s\tbar t}}) \bmu^\param_{s\tbar t}(\bx_t;\!\ddimstd)[\unmask] 
+ \scdhilight{\gamma_{s\tbar t}} \big( \hilight{\alpha_s \obs + \sigma_s \denoiser{s}{}{\bz_s}[\param][1][\unmask]}\big), \, 
\textcolor{purple}{\alpha_s^2 \sigma_\obs^2 \gamma_{s\tbar t}} \, \Id_{\dimobs}\Big) \eqsp,
\end{multline*}
where $\gamma_{s\tbar t} \eqdef \eta_s^2 / (\eta_s^2 + \alpha_s^2 \sigma_\obs^2)$. Hence, {\sc{MCGDiff}} differs from \algo\ on the choice of effective observation, which in this case is $\tilde\obs_s = \acp_s \obs$, the choice of variance in the transition and the coefficient of the convex combination. \rebuttal{From \eqref{eq:mcgdiff-transition1} it can be seen that MCGDIFF assumes the approximate model
$
\normpdf(\acp_s \obs; \bx_s[\unmask], \sigma^2 _{s\tbar \tau} \Id_\dimobs)
$
for the true likelihood $\pot{s}{\bx_s}$ \eqref{eq:guidance-term}.}

\paragraph{Comparison with the transition in \cite{zhu2023denoising} and \cite{martin2025pnpflow}.} We now write explicitely the algorithm \pnpflow\ \cite[Algorithm 3]{martin2025pnpflow} adapted to the inpainting problem we consider; see \Cref{algo:pnpflow}. We have simply adapted the notations and used $F(\bx) = \| \obs - \bx[\unmask] \|^2 / (2 \std^2 _\obs)$ in \citet[Algorithm 3]{martin2025pnpflow}. Thus, the transition used in \Cref{algo:pnpflow} is 
\begin{align*}
    \hpost{s\tbar t}{\bx_t}{\bx_s}[\param] & \propto \normpdf(\bx_s[\mask], \acp_s \denoiser{t}{}{\bx_t}[\param]{[\mask]}, \std^2 _s \Id_{\dimx - \dimobs}) \\
    & \hspace{1cm} \times \normpdf\left(\bx_s[\unmask], \left(1 - \scdhilight{\frac{\gamma_s}{\std^2 _\obs}}\right) \acp_s \denoiser{t}{}{\bx_t}[\param]{[\unmask]} + \scdhilight{\frac{\gamma_s}{\std^2 _\obs}} \hilight{\acp_s \obs},  \textcolor{purple}{\std^2 _s} \Id_{\dimx - \dimobs}\right) \eqsp.
\end{align*}
In the case of the DDIM schedule $\ddimstd_s = \sigma_s$, we have that $\mu^\param _{s\tbar t}(\bx_t) = \acp_s \denoiser{t}{}{\bx_t}[\param]$, and the MCGDIFF transition \citep{cardoso2023monte} in \eqref{eq:mcgdiff-transition} writes 
\begin{align*}
    \hpost{s\tbar t}{\bx_t}{\bx_s}[\param] &  = \normpdf(\bx_s[\mask], \acp_s \denoiser{t}{}{\bx_t}[\param]{[\mask]}, \std^2 _s \Id_{\dimx - \dimobs}) \\
    & \hspace{1cm} \times \normpdf(\bx_s[\unmask], (1 - \scdhilight{\tilde\gamma_{s\tbar t}})\acp_s \denoiser{t}{}{\bx_t}[\param]{[\unmask]} + \scdhilight{\tilde\gamma_{s\tbar t}} \hilight{\acp_s \obs},  \textcolor{purple}{\std^2 _{s\tbar \tau} \tilde\gamma_{s\tbar t}} \Id_{\dimx - \dimobs}). 
\end{align*}
Hence, the main difference lies in the coefficient of the convex combination and the variance used. 
\begin{algorithm}
    \caption{\pnpflow\ reinterpreted}
    \begin{algorithmic}[1]
        \STATE {\bfseries Input:} Decreasing timesteps $(t_k)_{k=K}^0$ with $t_K = 1$, $t_0 = 0$; adaptive stepsizes $(\gamma_k)_{k = K} ^0$.
        \STATE {\bfseries Initialize:} $\hat\bx_0 \in \mathbb{R}^d$.
        \FOR{$k = K-1$ {\bfseries to} $1$}
            \STATE $\tcz{\hat\bx_0[\unmask] \gets (1 - \frac{\gamma_k}{\std^2 _\obs}) \hat\bx_0[\unmask]  + \frac{\gamma_k}{\std^2 _\obs} \obs}$ 
            \STATE $\bw \sim \gauss(\zero_\dimx, \Id_\dimx)$
            \STATE $\bx \gets \acp_\tk{k} \hat\bx_0 + \std_\tk{k} \bw$
            \STATE $\hat\bx_0 \gets \denoiser{\tk{k}}{}{\bx}[\param]$
        \ENDFOR
        \STATE {\bfseries Return:} $\hat\bx_0$
    \end{algorithmic}
    \label{algo:pnpflow}
\end{algorithm}
\paragraph{Comparison with the transition in \cite{kim2025flowdps, patel2024flowchef}.} Here we explicitely write the transition of \flowdps\ for the inpainting case in order to understand the main differences without our method. For this purpose we rewrite \citet[Algorithm 1]{kim2025flowdps} using our notations. We note that the algorithm is written for the linear schedule $\acp_t = 1 - t$, $\sigma_t = t$ and the choice of DDIM schedule $\ddimstd_t = \std_t \sqrt{1 - \std_t}$, but we still write it with general notations to streamline the comparison with \Cref{algo:decoupled}. We also assume for the sake of simplicity that the optimization problem is solved exactly in \citet[line 7]{kim2025flowdps} (since there is no decoder as we solve the inverse problem in the latent space). The algorithm is given in \Cref{algo:flowdps}. In the specific setting where the linear schedule is used, setting $\gamma_k = \std^2 _\obs \std_\tk{k}$  in \Cref{algo:pnpflow} recovers \Cref{algo:flowdps} when $\ddimstd_k = \std_\tk{k}$.  Finally, we note that \flowdps\ can be understood as a noisy version of the \flowchef\ algorithm \citep{patel2024flowchef} and overall, follows the line of work of methods that learn a residual that is then used to translate the denoiser \citep{bansal2023universal,zhu2023denoising}. 
\begin{algorithm}
    \caption{\flowdps\ reinterpreted}
    \begin{algorithmic}[1]
        \STATE {\bfseries Input:} decreasing timesteps $(t_k)_{k=K}^0$ with $t_K = 1$, $t_0 = 0$; original image $\bx_\ast$; mask $\mask$; \\
       DDIM parameters $(\eta_k)_{k = K} ^0$.
       \STATE $\obs \gets \bx_\ast[\unmask]$
       \STATE $\bx \sim \gauss(\zero, \Id_\dimx)$
        \FOR{$k = K-1$ {\bfseries to} $0$}
            \STATE $\hat\bx_0 \gets \bx^\param _0(\bx, \tk{k+1})$
            \STATE $\hat\bx_1 \gets (\bx - \acp_\tk{k+1} \hat\bx_0) / \std_\tk{k+1}$
            \STATE \tcz{$\hat\bx_0[\unmask] \gets \acp_\tk{k} \hat\bx_0[\unmask] + \std_\tk{k} \obs$}
            \STATE $\bmu \gets \acp_\tk{k} \hat\bx_0 + (\std^2 _\tk{k} - \ddimstd^2 _{k})^{1/2} \hat\bx_1$ 
            \STATE $\bw \sim \gauss(\zero_\dimx, \Id_\dimx)$
            \STATE $\bx \gets \bmu + \eta_k \bw$
        \ENDFOR
        \STATE {\bfseries Return:} $\bx$
    \end{algorithmic}
    \label{algo:flowdps}
\end{algorithm}
\paragraph{Comparison with DiffPIR \citep{zhu2023denoising} and DDNM \citep{wang2023zeroshot}.} \rebuttal{We provide the \diffpir\ algorithm \cite[Algorithm 1]{zhu2023denoising} adapted to our inpainting case  using our own notation in \Cref{algo:diffpir}. In Line~\ref{line:diffpir-opt} we write the exact solution to the optimization problem in the original algorithm. We write the associated transition in a convenient form that allows a seamless comparison with our algorithm. Define $\gamma_t \eqdef \std^2 _t / (\std^2 _t + \lambda \acp^2 _t \std^2 _\obs)$. Then, the transition}
\begin{multline*}    
\hpost{s\tbar t}{\bx_t, \obs}{\bx_s}[\param] 
= \normpdf
\big(\bx_s[\mask]; \bmu^\param_{s\tbar t}(\bx_t;\!\ddimstd)[\mask], \, \ddimstd^2_s \Id_{\dimx - \dimobs}\big) \\
\times \normpdf 
\Big(\bx_s[\unmask]; (1 - \scdhilight{\gamma_{t}}) \bmu^\param_{s\tbar t}(\bx_t;\!\ddimstd)[\unmask] 
+ \scdhilight{\gamma_{t}} \big( \hilight{\alpha_s \obs + (\sigma^2 _s - \ddimstd^2 _s)^{1/2} \frac{\bx_t[\unmask] - \acp_t \obs}{\std_t}}\big), \, 
\textcolor{purple}{\ddimstd^2 _s} \, \Id_{\dimobs}\Big) \eqsp,
\end{multline*}
\rebuttal{corresponds to one step of \Cref{algo:diffpir}. We highlight key distinctions:
\begin{itemize}[leftmargin=*, labelsep=0.5em]
    \item  Setting $\ddimstd_s^2 = \std_s^2$ recovers the same transition as in \pnpflow.
    \item The main distinction lies in the mean of the Gaussian transition for the unmasked region: it is a convex combination of $\bmu^\param(\bx_t;\ddimstd)[\unmask]$ and an effective observation $\alpha_s \obs + (\sigma_s^2 - \ddimstd_s^2)^{1/2} (\bx_t[\unmask] - \acp_t \obs)/\std_t$. In our algorithm, the effective observation instead takes the form $\acp_s \obs + \std_s \noisepred{s}{}{\bx_s}[\param][\unmask]$. We estimate the residual noise using the pre-trained model at timestep $s$, whereas \diffpir\ computes it as $(\bx_t[\unmask] - \acp_t \obs)/\std_t$.
    \item This residual noise is scaled differently: by $(\std_s^2 - \ddimstd_s^2)^{1/2}$ in \diffpir, and by $\std_s$ in our method.
    \item The convex combination coefficient in our cases is $\gamma_{s\tbar t} = \ddimstd^2 _s / (\ddimstd^2 _s + \acp^2 _s \std^2 _\obs)$ whereas for \diffpir\ it is set to $\gamma_t = \std^2 _t / (\std^2 _t + \acp^2 _t \std^2 _\obs)$. 
    \item Finally, the noise-free ($\std_\obs = 0$) version of \diffpir\ recovers the DDNM algorithm \citep{zhang2023towards}. 
\end{itemize}
} 
\begin{algorithm}
    \caption{\diffpir\ reinterpreted}
    \begin{algorithmic}[1]
        \STATE {\bfseries Input:} Decreasing timesteps $(t_k)_{k=K}^0$ with $t_K = 1$, $t_0 = 0$; scaling $\lambda$; original image $\bx_\ast$; mask $\mask$; DDIM parameters $(\eta_k)_{k = K} ^0$
        \STATE $\obs \gets \bx_\ast[\unmask]$
        \STATE $\bx \sim \gauss(\zero, \Id_\dimx)$.
        \FOR{$k = K-1$ {\bfseries to} $1$}
            \STATE $\hat\bx_0 \gets \bx^\param _0(\bx, \tk{k+1})$
            \STATE \tcz{$\hat\bx_0[\unmask] \gets \frac{\std^2 _\tk{k+1}}{\std^2 _\tk{k+1} + \lambda \std^2 _\obs \acp^2 _\tk{k+1}} \obs + \frac{\lambda \std^2 _\obs \acp^2 _\tk{k+1}}{\std^2 _\tk{k+1} + \lambda \std^2 _\obs \acp^2 _\tk{k+1}} \hat\bx_0[\unmask]$} \label{line:diffpir-opt}
            \STATE $\hat\bx_1 \gets (\bx - \acp_\tk{k+1} \hat\bx_0) / \std _\tk{k+1}$
            \STATE $\bw \sim \gauss(\zero_\dimx, \Id_\dimx)$
            \STATE $\bx \gets \acp_\tk{k} \hat\bx_0 + (\std^2 _\tk{k} - \eta^2 _k)^{1/2} \hat\bx_1 + \eta_k \bw$ \label{line:diffpir-ddim}
        \ENDFOR
        \STATE {\bfseries Return:} $\bx$
    \end{algorithmic}
    \label{algo:diffpir}
\end{algorithm}
\paragraph{Further related methods.} 
Here we continue our discussion of VJP-free methods. The \daps\ algorithm \citep{zhang2025improving} proposes sampling, given the previous state $X_\tk{k+1}$, a clean state $\hat{X}_0$ by performing Langevin Monte Carlo steps on the posterior distribtion $\smash{\post{0\tbar \tk{k+1}}{X_\tk{k+1}, \obs}{}}$. This step is performed approximately by replacing the prior transition $\smash{\pdata{0\tbar \tk{k+1}}{X_\tk{k+1}}{}}$ with a Gaussian approximation centered at the denoiser $\smash{\denoiser{\tk{k+1}}{}{X_\tk{k+1}}[\param]}$. Then, given $\hat{X}_0$, the next state $X_\tk{k}$ is drawn from $\smash{\gauss(\acp_\tk{k} \hat{X}_0, \std^2 _\tk{k} \Id_\dimx)}$.

One important aspect of our method is that we circumvent differentiation through the denoiser but also the decoder, as the diffusion models we consider operate in the latent space. We do so by downsampling the mask into the latent space. In contrast, the recent work of \citet{spagnoletti2025latino} also circumvents differentiation through the denoiser, but does so by lifting the latent states into pixel space and optimizing the likelihood there. The result of the optimization is then projected back into the latent space and then undergoes back-and-forth noise-denoising steps. 

Finally, several recent works \citep{mardani2024a, zilberstein2025repulsive, erbach2025solving} adopt a variational perspective: the target distribution is approximated by a Gaussian distribution whose parameters are iteratively estimated by minimizing a combination of an observation-fidelity loss and a score-matching-like loss.

\emph{VJP-based methods.} A broad class of zero-shot approaches builds on the guidance approximation \eqref{eq:dps} to estimate $\nabla_{\bx_t} \log \pot{t}{\bx_t}$. \citet{song2022pseudoinverse} approximate $\pdata{0\tbar t}{}{}$ by a Gaussian with mean $\denoiser{t}{}{}[\param]$ and a tuned covariance. For the inpainting setting in \eqref{eq:inpainting-likelihood}, plugging this approximation into \eqref{eq:guidance-term} yields an integral that can be computed in a closed form, providing a proxy for $\pot{t}{\cdot}$. Several works exploit the link between the covariance of $\pdata{0\tbar t}{\bx_t}{\cdot}$ and the Jacobian of the denoiser \citep{meng2021estimating}. This observation underpins the methods of \citet{finzi2023user}, \citet{stevens2023removing}, and \citet{boys2023tweedie}, which derive likelihood scores by estimating or inverting the Jacobian. These approaches require solving large linear systems and backpropagating through the denoiser, both computationally expensive operations. To reduce cost, these works   assume a locally constant Jacobian around $\bx_t$, but updates still involve either explicit matrix inversion or repeated VJPs. In practice, diagonal approximations based on row sums are commonly used to approximate the covariance matrix \citep{boys2023tweedie}, or conjugate gradient methods are employed to circumvent the need for full matrix inversion \citep{rozet2024learning}. For general likelihoods $\pot{0}{\cdot}$, \citet{song2023loss} combine the Gaussian posterior model of \citet{song2022pseudoinverse} with Monte Carlo sampling to approximate $\pot{t}{\cdot}$. In the latent setting, \citet{rout2024solving} apply the DPS approximation jointly with a regularizer that encourages latent variables to remain near encoder–decoder fixed points.  Other methods modify the sampling dynamics. \citet{moufad2025variational} propose a two-stage procedure: the chain is first moved to an earlier time $\ell \ll \tk{k}$, where the DPS approximation is applied to sample from an approximate conditional at $\ell$, before returning to $\tk{k}$ via additional noising steps. \citet{janati2025mgdm} incorporate a related idea into a Gibbs sampling framework.  Overall, these methods remain fundamentally VJP-based and inherit substantial memory and runtime overhead from repeated backpropagation through the denoiser. By contrast, our decoupled guidance relies exclusively on forward denoiser evaluations and closed-form Gaussian updates, thereby eliminating VJPs entirely while retaining competitive performance.

For a complete review of zero-shot posterior sampling methods see \cite{daras2024survey,janati2025bridging,chung2025diffusion}.

\subsection{Behavior under increased runtime.}
We extend the ablation study in \Cref{subsec:ablation} by examining the behavior of \algo\  when the number of NFEs is increased. Specifically, we vary the budget from 20 to 500 NFEs on the DIV2K dataset and report results across different masking patterns; see \Cref{fig:runtime}. All metrics improve steadily as the budget grows, reaching their best values around 200 NFEs ($10$s runtime). Beyond this point, performance saturates and exhibits a slight degradation at 500 NFEs. These results suggest that our default DDIM schedule is well suited to low and mid-NFE regimes---which are most relevant for practical settings---but may not be fully optimized for larger budgets.
\begin{figure}[H]
    \centering
    \includegraphics[width=.80\textwidth]{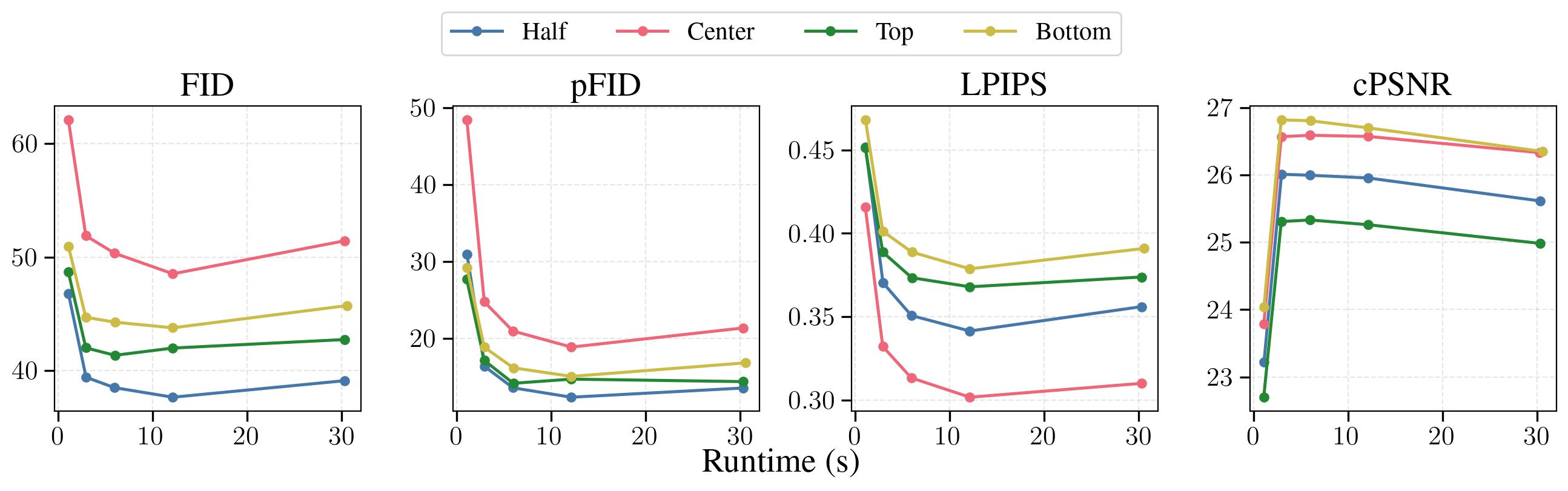}
    \captionsetup{font=small}
    \caption{Performance of DING on DIV2K under varying NFE budgets (20 to 500) across different masking patterns. Runtimes are measured on a H100 GPU.}
    \label{fig:runtime}
\end{figure}

\subsection{Limitation}
\rebuttal{We observed that the quality of reconstructions is highly sensitive to the specificity of the textual prompt. When the prompt is under-specified or lacks sufficient semantic detail, the resulting samples may exhibit reduced coherence, particularly in large masked regions where contextual consistency is critical. This issue manifests as mismatched textures or backgrounds, or inconsistent object boundaries, even when the visible area is faithfully preserved. To illustrate this behavior, we compare reconstructions obtained with well-defined prompts against those generated using vague or ambiguous ones. Examples are provided in Figure~\ref{fig:limitation1} and \ref{fig:limitation2}.}
\begin{figure}
    \centering 
    \includegraphics[width=0.9\textwidth]{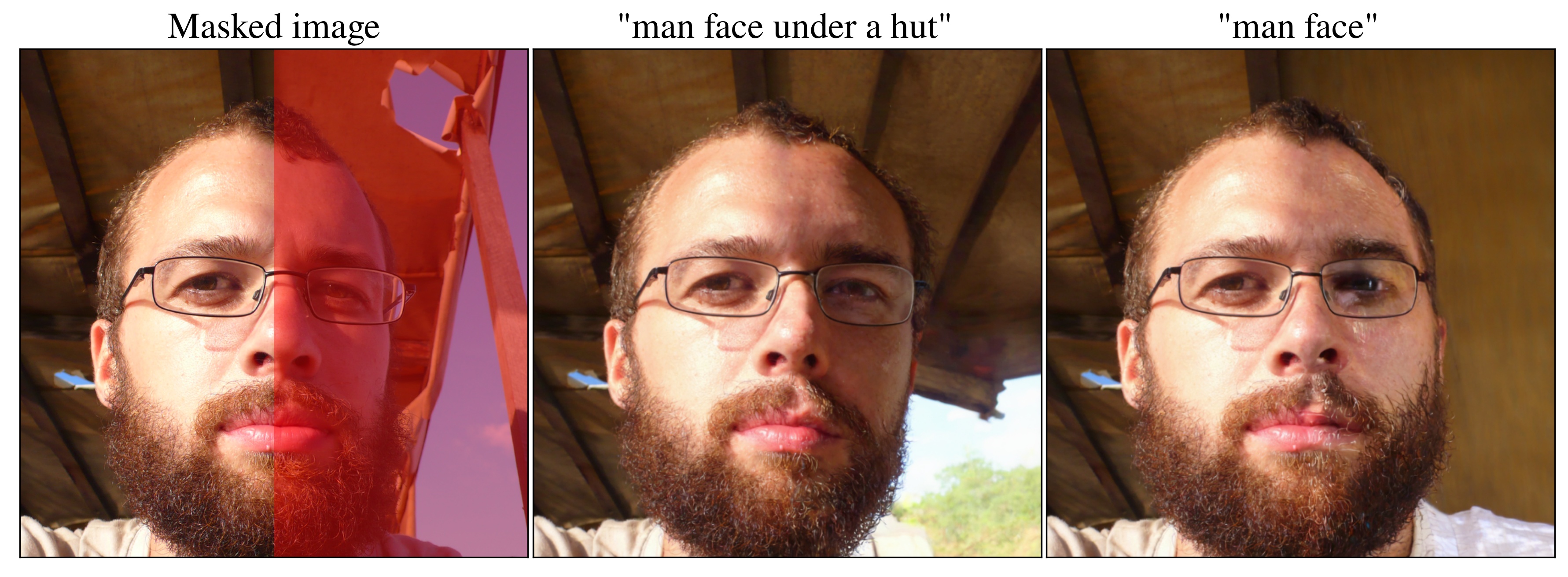}
    \includegraphics[width=0.9\textwidth]{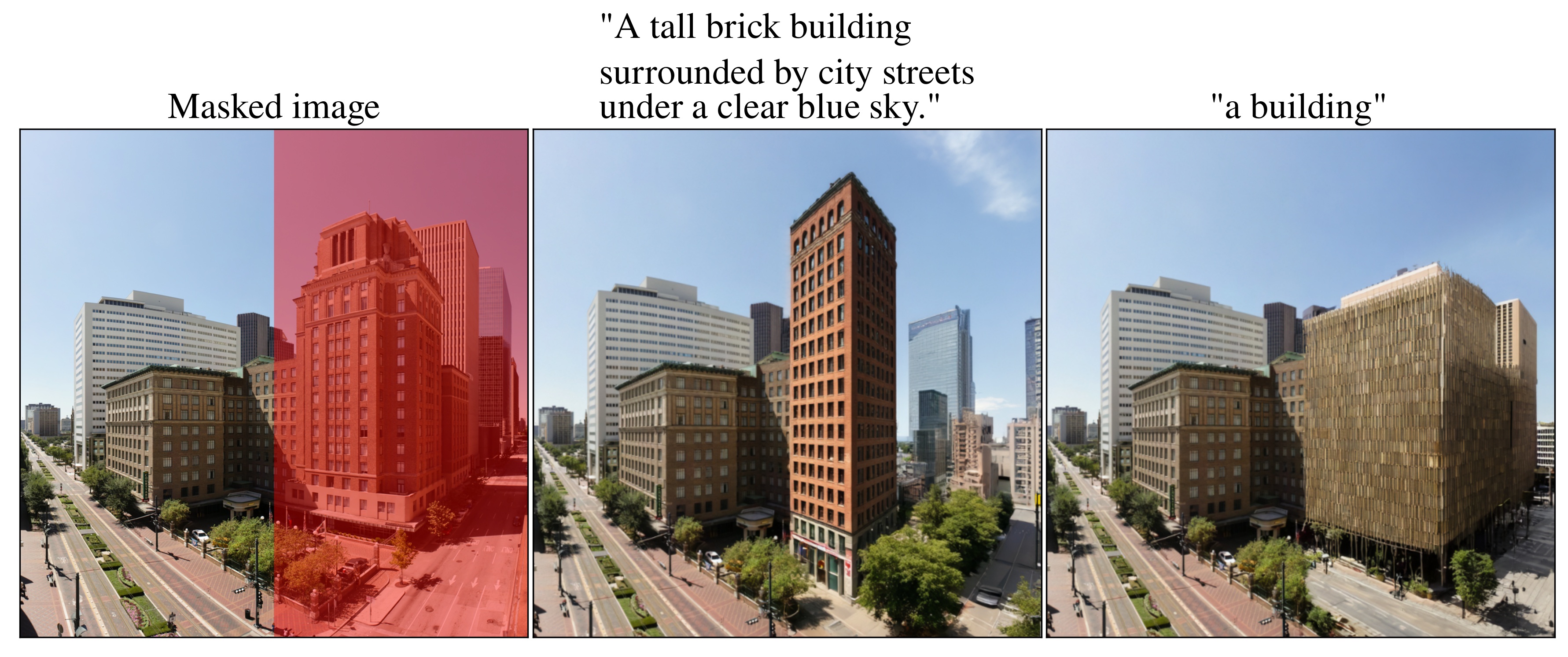}
    \caption{Effect of prompt precision on inpainting quality.}
    \label{fig:limitation1}
\end{figure}
\begin{figure}
    \centering
    \includegraphics[width=0.9\textwidth]{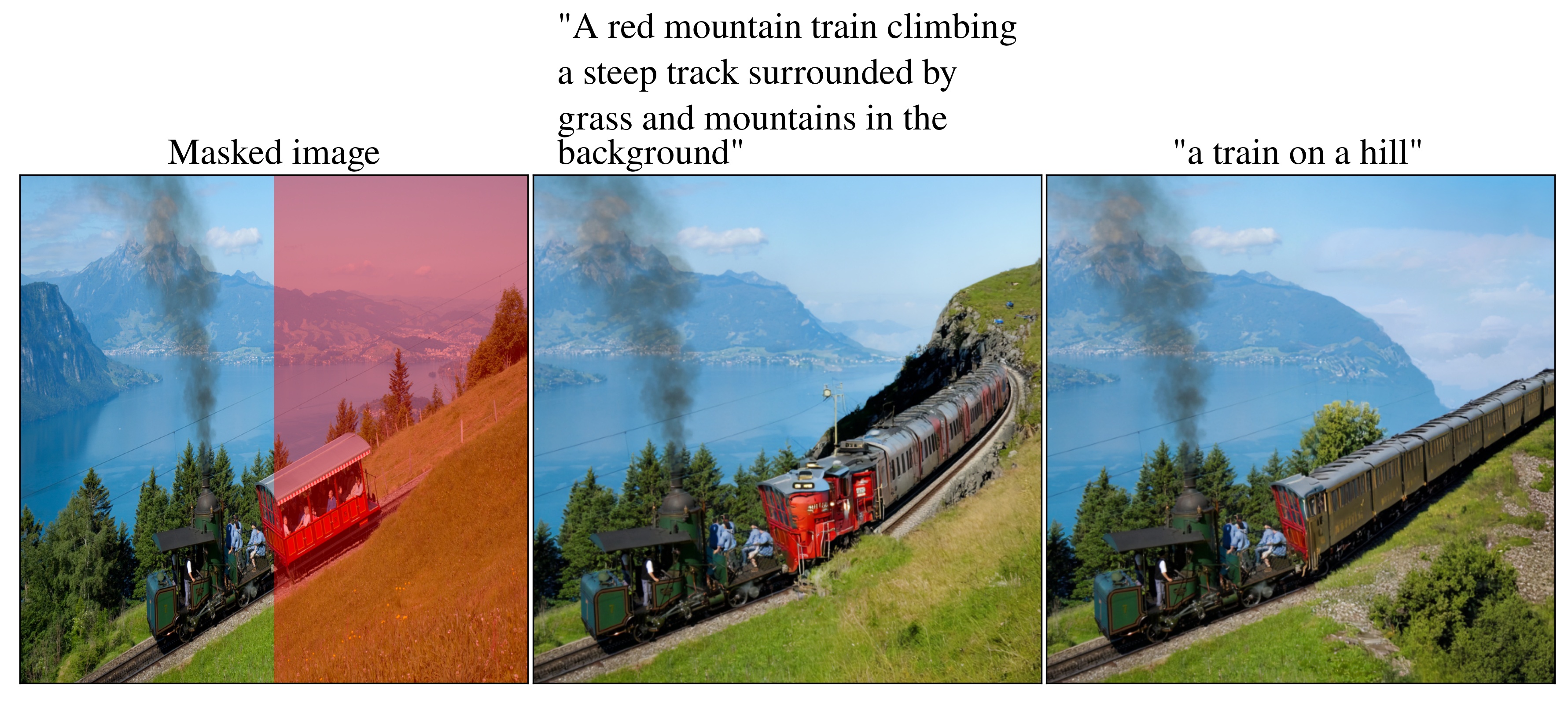}
    \includegraphics[width=0.9\textwidth]{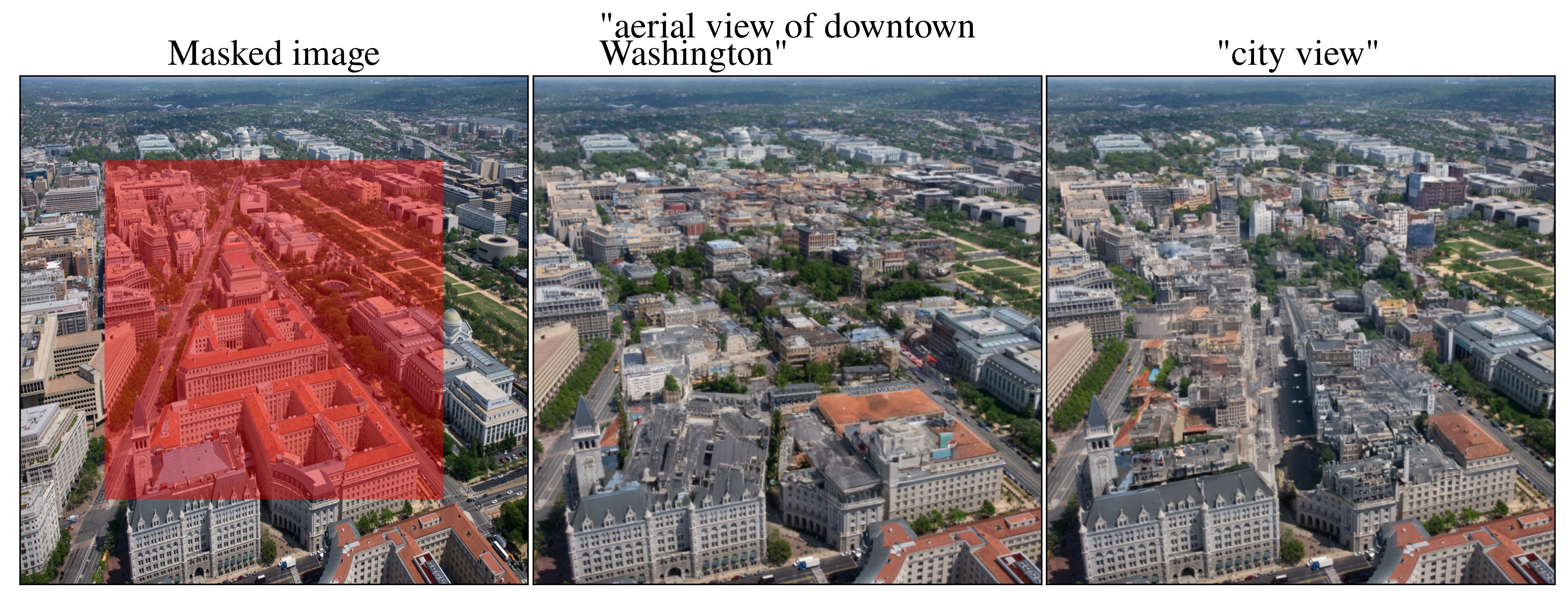}
    \caption{Effect of prompt precision on inpainting quality.}
    \label{fig:limitation2}
\end{figure}
\subsection{Bias in Gaussian case}
\rebuttal{For the sake of simplicity we assume that $\pdata{0}{}{} \eqdef \gauss(\zero_\dimx, \Sigma)$ where $\Sigma$ is a covariance matrix. We also write the likelihood as $\pot{0}{\bx_0} = \normpdf(\obs; P_{\unmask} \bx_0, \sigma^2 _\obs \Id_\dimobs)$ where $P_\unmask \in \rset^{\dimobs \times \dimx}$ is the matrix satisfying $P_{\unmask} \bx = \bx[\unmask]$. Define $D_t \eqdef \acp_t \Sigma(\acp^2 _t \Sigma + \sigma^2 _t \Id_\dimx)^{-1}$. Then, the denoiser and noise predictors are given by 
\begin{align*}
\denoiser{t}{}{\bx_t} = D_t \bx_t \eqsp, \quad 
\noisepred{t}{}{\bx_t} = \sigma^{-1} _t \left( \Id_\dimx - \acp_t D_t \right) \bx_t \eqsp.
\end{align*}
We consider hereafter the DDIM transitions 
$
\pdata{s\tbar t}{\bx_t}{\bx_t}[\ddimstd] \eqdef \normpdf(\bx_s; \bmu_{s\tbar t}(\bx_t; \ddimstd), \ddimstd^2 _s \Id_\dimx) 
$
where 
$$
    \bmu_{s\tbar t}(\bx_t; \ddimstd) \eqdef \acp_s \denoiser{t}{}{\bx_t} + \sqrt{\sigma^2 _s - \ddimstd^2 _s} \noisepred{t}{}{\bx_t} 
$$
In this section we analyze the bias of the \algo\ one-step transition relative to the posterior transition involving the DPS likelihood \eqref{eq:dps}; \emph{i.e.} we compare the transition 
\begin{equation}
\hpost{s\tbar t}{\bx_t, \obs}{\bx_s}[\ding]
\;\eqdef\; \pE \left[ \hpost{s\tbar t}{Z_s, \bx_t, \obs}{\bx_s}[\ding]  \right] \eqsp,
\end{equation}
where $Z_s \sim \pdata{s\tbar t}{\bx_t}{}[\ddimstd]$ and 
$$
\hpost{s\tbar t}{\bz_s, \bx_t, \obs}{\bx_s}[\ding] \propto \pot{0}{\frac{\bx_s - \sigma_s \noisepred{s}{}{\bz_s}}{\acp_s}} \pdata{s\tbar t}{\bx_t}{\bx_s}[\ddimstd]
$$  against 
\begin{equation*}
    \hpost{s\tbar t}{\bx_t, \obs}{\bx_s}[\mathsf{dps}] \propto \pot{0}{\denoiser{s}{}{\bx_s}} \pdata{s\tbar t}{\bx_t}{\bx_s}[\ddimstd] \eqsp.
\end{equation*}}
\rebuttal{
We define $M \eqdef P_\unmask^\top P_\unmask$, which is an orthogonal projection matrix since $M^\top = M$, $P_\unmask P_\unmask^\top = \Id_\dimobs$, and thus $M^2 = M$. 
We also introduce the quantity 
\[
\varepsilon_s \eqdef \| (D_s^\top - \acp_s^{-1} \Id_\dimx) M \|_{\mathrm{op}},
\]
which quantifies how far the Jacobian of the denoiser $\denoiser{s}{}{}$ deviates from the Jacobian of the \algo\ denoiser approximation \emph{on the observed coordinates}. 
In the following proposition, we characterize the asymptotic behavior of the DPS and \algo\ posterior transition means and covariances as $\ddimstd_s \to 0$, and we express the mean bias in terms of $\varepsilon_s$. 
In \Cref{prop:upperbound_eps}, we also provide an explicit upper bound on $\varepsilon_s$ in terms of the schedule and the minimum eigenvalue of the prior covariance $\Sigma$.}
\begin{proposition}
\rebuttal{Both $\hpost{s\tbar t}{\bx_t,\obs}{}[\mathsf{dps}]$ and $\hpost{s\tbar t}{\bx_t,\obs}{}[\ding]$ are Gaussian distributions with mean and covariance respectively $(\bmu^\dps _{s\tbar t}(\bx_t, \obs), \Sigma^\dps _{s\tbar t})$ and $(\bmu^\ding _{s\tbar t}(\bx_t, \obs), \Sigma^\ding _{s\tbar t})$ satisfying 
\begin{align*} 
    \| \Sigma^\dps _{s\tbar t} - \Sigma^\ding _{s\tbar t} \| = \bigO(\ddimstd^4 _s)
\end{align*}
and 
$$
    \| \bmu^\dps _{s\tbar t}(\bx_t, \obs) - \bmu^\ding _{s\tbar t}(\bx_t, \obs) \| = \bigO\bigg(\ddimstd^2 _s \big(\varepsilon_s (\| \obs \| + \| M \bmu_{s\tbar t}(\bx_t; \ddimstd) \| \big) + \varepsilon^2 _s \| \bmu_{s\tbar t}(\bx_t; \ddimstd) \| \big) \bigg) \eqsp. 
$$
as $\ddimstd_s \rightarrow 0$.}
\end{proposition}
\begin{proof} 
\rebuttal{Using the standard Gaussian conjugation formula \cite[equation 2.116]{bishop2006pattern}, we have that $\hpost{s\tbar t}{\bx_t, \obs}{\bx_s}[\mathsf{dps}] = \normpdf(\bx; \mathbf{m}^\dps _{s\tbar t}(\bx_t, \obs), \Sigma^\dps _s)$ with 
\begin{align*}
    \mathbf{m}^\dps _{s\tbar t}(\bx_t, \obs) & \eqdef \Sigma^\dps _{s\tbar t}(\ddimstd^{-2} _s \bmu_{s\tbar t}(\bx_t; \ddimstd) + \std^{-2} _\obs D^\top _s P^\top _\unmask \obs) \eqsp, \\
    \Sigma^\dps _{s\tbar t} & \eqdef \big(\ddimstd^{-2} _s \Id_\dimx + \std^{-2} _\obs (P_\unmask D_s)^\top P_\unmask D_s\big)^{-1} \eqsp.
\end{align*}
Next, for the \algo\ transition, first set 
$b_s(Z_s)\eqdef -(\std_s/\acp_s) P_{\unmask}\,\noisepred{s}{}{Z_s}$.  Gaussian conjugacy with
$\pdata{s\tbar t}{\bx_t}{\bx_s}[\ddimstd] = \gauss(\bx_s;\bmu_{s\tbar t}(\bx_t;\ddimstd),\ddimstd^2 _s  \Id_\dimx)$
gives
$$
\hpost{s\tbar t}{Z_s,\bx_t,\obs}{\bx_s}[\ding]
=\normpdf\big(\bx_s; \widetilde\Sigma^{\ding}_{s\tbar t}\big(\ddimstd_s^{-2}\,\bmu _{s\tbar t}(\bx_t;\ddimstd)+ \std_\obs^{-2} \acp^{-1} _s P^\top _\unmask \big(\obs - b_s(Z_s)\big)
\big), \widetilde\Sigma^{\ding} _{s\tbar t}
\Big),
$$
and $\widetilde\Sigma^{\ding} _{s\tbar t} \eqdef \big(\ddimstd_s^{-2}\Id_\dimx+\acp_s^{-2}\std_\obs^{-2} P_{\unmask}^\top P_{\unmask}\big)^{-1} \eqsp.$ Since the mean of this conditional distribution is clearly affine in $Z_s$, we integrate it out, yielding 
that $\hpost{s\tbar t}{\bx_t, \obs}{\bx_s}[\ding] = \normpdf(\bx_s; \mathbf{m}^\ding _{s\tbar t}(\bx_t, \obs), \Sigma^\ding _{s\tbar t})$ where 
\begin{align*} 
    \mathbf{m}^\ding _{s\tbar t}(\bx_t, \obs) & \eqdef \widetilde\Sigma^{\ding}_{s\tbar t}\big(\ddimstd_s^{-2}\,\bmu _{s\tbar t}(\bx_t;\ddimstd) + \std_\obs^{-2} \acp^{-1} _s P^\top _\unmask \obs \\
    & \hspace{2cm} + (\std_\obs^{-2} \acp^{-2} _s)  P_\unmask^\top P_\unmask (\Id_\dimx - \acp_s D_s) \bmu _{s\tbar t}(\bx_t;\ddimstd) \big) \eqsp, \\
    \Sigma^\ding _{s\tbar t} & \eqdef \widetilde\Sigma^{\ding}_{s\tbar t} +   \frac{\ddimstd^2 _s}{\std_\obs^{4} \acp^4 _s} \widetilde\Sigma^{\ding}_{s\tbar t} P^\top _\unmask P _\unmask (\Id_\dimx - \acp_s D_s) (\widetilde\Sigma^{\ding}_{s\tbar t} P^\top _\unmask P _\unmask (\Id_\dimx - \acp_s D_s))^\top \eqsp.
\end{align*}
}
\paragraph{Small-noise regime.}
\rebuttal{We now study the behavior of both transitions when the DDIM kernel variance $\eta_s^2$ tends  to zero. For simplicity we define 
$$ 
    K_\dps \eqdef \std^{-2} _\obs D^\top _s P^\top _\unmask P_\unmask D_s \eqsp, \quad K_\ding \eqdef \acp^{-2} _s \std^2 _\obs M 
$$ 
and $R_s = \Id_\dimx - \acp_s D_s$, $M = P^\top _\unmask P _\unmask$. Then, 
\begin{align*}
\Sigma^{\dps}_{s|t} & = (\eta_s^{-2} \Id_d + K_\dps)^{-1} \eqsp, \\
\Sigma^{\ding}_{s|t} & = (\eta_s^{-2} \Id_d + K_\ding)^{-1} + \frac{\eta_s^2}{\alpha_s^4 \sigma_y^4}(\eta_s^{-2} \Id_d + K_\ding)^{-1} M R_s R_s^\top M (\eta_s^{-2} \Id_d + K_\ding)^{-1}.
\end{align*}
We use throughout that for any fixed matrix $K$, we have that when $\eta_s^2 \|K\|_{\mathrm{op}} < 1$,
\begin{equation}
(\eta_s^{-2} \Id_\dimx + K)^{-1} = \eta_s^2 (\Id_\dimx - \eta_s^2 K) + R_2(\eta_s),
\qquad
\|R_2(\eta_s)\| \le \eta_s^6 \frac{\|K\|^2 _{\mathrm{op}}}{1 - \eta_s^2 \|K\|_{\mathrm{op}}}.
\label{eq:neumann}
\end{equation}
This follows from the standard Neumann (geometric) series expansion. Applying~\eqref{eq:neumann} with $\eta^2 _s \leq \min(1 / \| K_\dps \|_{\mathrm{op}}, 1 / \| K_\ding \|_{\mathrm{op}})$, we get 
\begin{align*}
\Sigma^{\dps}_{s\tbar t}  &= \ddimstd_s^2 (\Id_d - \ddimstd_s^2 K_{\dps}) + \bigO(\ddimstd_s ^6), \\
\Sigma^{\ding}_{s\tbar t} &= \ddimstd_s^2 (\Id_d - \ddimstd_s^2 K_{\ding}) + \frac{\ddimstd_s^6}{\acp_s^4 \std_y^4} M R_s R_s^\top M + \bigO(\ddimstd_s^6).
\end{align*}
and thus 
$$ 
 \Sigma^{\dps}_{s\tbar t} - \Sigma^{\ding}_{s\tbar t} = \bigO(\ddimstd^4 _s) \eqsp. 
$$ 
Plugging these expansions in the mean terms, we find that 
\begin{align*}
\mathbf{m}^{\dps}_{s\tbar t}(\bx_t, \obs) &= \bmu_{s\tbar t}(\bx_t; \ddimstd) + \eta_s^2 (\sigma_y^{-2} D_s^\top P_\unmask ^\top \obs - K_{\dps} \bmu_{s\tbar t}(\bx_t; \ddimstd)) + \bigO(\ddimstd_s^4), \\
\mathbf{m}^{\ding}_{s\tbar t}(\bx_t, \obs) & = \bmu_{s\tbar t}(\bx_t; \ddimstd) + \ddimstd_s^2 (\acp_s^{-1} \std_y^{-2} P_\unmask ^\top \obs
+ \acp_s^{-2} \std_y^{-2} M R_s \bmu_{s\tbar t}
- K_{\ding} \bmu_{s\tbar t}(\bx_t; \ddimstd)) + \bigO(\ddimstd_s^4).
\end{align*}
}
\rebuttal{This yields  
\begin{multline*}
 \mathbf{m}^{\dps}_{s\tbar t}(\bx_t, \obs) - \mathbf{m}^{\ding}_{s\tbar t}(\bx_t, \obs) = 
\eta_s^2 \sigma_y^{-2}
\big[
(\alpha_s^{-1} \Id_d - D_s^\top) P_\unmask ^\top \obs
\\ + \alpha_s^{-2} M R_s \bmu_{s\tbar t}(\bx_t; \ddimstd)
- (\alpha_s^{-2} M - D_s^\top M D_s) \bmu_{s\tbar t}(\bx_t; \ddimstd)
\big] + \bigO(\eta_s^4),
\end{multline*}
We now proceed to further upper bound the leading term. Define $E_s \eqdef D_s - \acp^{-1} _s \Id_\dimx$. Then $R_s = - \acp_s E_s$ and we have that
\begin{multline*}
     \mathbf{m}^{\dps}_{s\tbar t}(\bx_t, \obs) - \mathbf{m}^{\ding}_{s\tbar t}(\bx_t, \obs) = \\ \ddimstd^2 _s \std^{-2} _\obs \big( -E^\top _s P^\top _\unmask \obs + \acp^{-1} _s E^\top _s M \bmu_{s\tbar t}(\bx_t; \ddimstd) + E^\top _s M E_s \bmu_{s\tbar t}(\bx_t; \ddimstd) \big) + \bigO(\ddimstd^4 _s) \eqsp.
\end{multline*}
with $M=P_\unmask^\top P_\unmask$, which is an orthogonal projection matrix since $M^\top = M$ and $P_\unmask P^\top _\unmask = \Id_\dimobs$ and thus $M^2 = M$. We proceed by bouding each term of 
$$
-E^\top _s P^\top _\unmask \obs + \acp^{-1} _s E^\top _s M \bmu_{s\tbar t}(\bx_t; \ddimstd) + E^\top _s M E_s \bmu_{s\tbar t}(\bx_t; \ddimstd)
$$
separately. Define
$
\varepsilon_s \eqdef \|M E_s\|_{\mathrm{op}}.
$ Then, since $\bv \eqdef P_\unmask^\top \obs \in \mathrm{range}(M)$, we have $M \bv = \bv$. Hence
$$
E_s^\top P_\unmask^\top y = E_s^\top M P_\unmask^\top y = (M E_s)^\top (M P_\unmask^\top \obs) 
$$
where we have used that $M^{\top} M = M$. 
By the operator norm inequality, and the fact that $\|P_\unmask^\top \obs\| = \|\obs\|$, we get 
\[
\|E_s^\top P_\unmask^\top \obs\|
\;\le\; \|M E_s\|_{\mathrm{op}}\, \|M P_\unmask^\top \obs\|
= \varepsilon_s \|P_\unmask^\top \obs\| = \varepsilon_s \| \obs\|.
\]
Next, using the same operator norm inequality we get that 
$$
\|E^\top _s M \bmu_{s\tbar t}(\bx_t; \ddimstd) \| \leq \varepsilon_s \| M \bmu_{s\tbar t}(\bx_t; \ddimstd) \| \eqsp, \quad \| E^\top _s M E_s \bmu_{s\tbar t}(\bx_t; \ddimstd) \| \leq \varepsilon^2 _s \| \bmu_{s\tbar t}(\bx_t; \ddimstd) \| \eqsp.
$$
which yields the desired bound. }
\end{proof} 
\begin{proposition}[Upperbound on $\varepsilon_s$]
    \label{prop:upperbound_eps}
\rebuttal{We have that 
$$
\varepsilon_s \leq  \frac{\sigma_s^2}{\alpha_s}\, \frac{1}{\alpha_s^2\,\lambda_{\min}(\Sigma)+\sigma_s^2}
$$
where $\lambda_{\min}(\Sigma)$ is the smallest eigenvalue of $\Sigma$. }
\end{proposition}
\begin{proof} 
\rebuttal{By noting that $(\acp^2 _s \Sigma + \std^2 _s \Id) E_s = - \acp^{-1} _s \std^2 _s \Id_\dimx$, we get the alternative expression 
$$ 
    E_s = - \frac{\std^2 _s}{\acp_s} (\acp^2 _s \Sigma + \std^2 _s \Id_\dimx)^{-1} \eqsp.
$$ 
By the submultiplicativity of the operator norm and the fact that $M$ is a non-trivial orthogonal projection matrix, we have that 
$$
    \| E^\top M \|_{\mathrm{op}} \leq \| E \|_{\mathrm{op}} = \frac{\std^2 _s}{\acp_s} \frac{1}{\lambda_{\min}(\acp^2 _s \Sigma + \std^2 _s \Id_\dimx)} \leq \frac{\std^2 _s}{\acp _s} \frac{1}{\acp^2 _s \lambda_{\min}(\Sigma) + \std^2 _s} \eqsp.
$$
}
\end{proof}

\section{Details about the experiments}
\label{sec:exp-details}

\subsection{Models}
We use both the SD 3 and SD 3.5 (medium) \citep{esser2024scaling} models with the linear schedule $\alpha_t = 1 - t$ and $\sigma_t = t$. 
In all the experiments we run the zero-shot methods with a guidance scale of $2$.  The fine-tuned baseline, which we refer to as SD3 Inpaint, is based on the publicly available model\footnote{\url{https://huggingface.co/alimama-creative/SD3-Controlnet-Inpainting}} trained for inpainting with a ControlNet-augmented version of Stable Diffusion 3. It has been finetuned on a large dataset of approximately 12 million $1024 \times 1024$ image–mask pairs to directly predict high-quality inpainted completions conditioned on the masked image and the mask itself. We have found the model to perform well also on lower resolutions, despite not undergoing multi-resolution training. Examples of image editing of lower resolution images are presented in the the HuggingFace page of the smae project. \rebuttal{We run this baseline using a guidance scale of $7$ for optimal results.}

Finally, all experiments use \texttt{bfloat16} for model forward passes \rebuttal{(and backward passes for baselines that require it)}, with other computations performed in \texttt{float32}.

\subsection{Mask downsampling}
To construct the mask in the latent space, we start from the original binary mask defined in pixel space. Since the encoder reduces spatial resolution by a fixed factor (here, 8), we downsample the pixel-space mask to match the resolution of the latent representation. This is done by applying bilinear interpolation with antialiasing. 
The resulting low-resolution mask captures the proportion of masked pixels within each latent receptive field. Finally, we threshold this downsampled mask at $0.95$ to obtain a binary latent mask, slightly overestimating the masked region to prevent boundary artifacts during sampling.

\subsection{Implementation of the baselines}

Here, we give implementation details of the baselines. We stress that \emph{each baseline is run in the latent space}, and thus no method computes the gradient \wrt\ the input of the decoder. We also manually tuned each baseline for the considered tasks. We provide the used hyperparameters in \Cref{tab:hyperparams-tasks}. 

\paragraph{\rebuttal{\blended.}}
\rebuttal{%
We implemented \citet[Algorithm~1]{avrahami2023blendedlatent} following their official code\footnote{\url{https://github.com/omriav/blended-latent-diffusion}}.
The codebase includes an additional hyperparameter, \texttt{blending\_percentage}, which determines at what fraction of the inference steps blending begins. We set it to zero, as applying blending across all steps produced the best results.
A key detail is the original implementation is that the observed region (background) is re-noised to the noise level defined by the current timestep; see \citet[step 1-2 within the for loop in Algo 1]{avrahami2023blendedlatent}, yet the reconstructed region (foreground) has less noise as it comes from applying a DDIM transition. This causes the background and foreground to follow different noise levels, and hence, introduces minor artifacts in the final reconstructions.
We fixed this issue in our implementation by matching the two noise levels.
}

\paragraph{\daps.}
We adapt \citet[Algorithm 1]{zhang2025improving} based on the released code\footnote{\url{https://github.com/zhangbingliang2019/DAPS}} to the flow matching formulation. We found that using Langevin as MCMC sampler for enforcing data consistency works the best for low NFE regime.

\paragraph{\rebuttal{\diffpir.}}
\rebuttal{%
We make \citet[Algorithm 1]{zhu2023denoising} compatible with the flow matching formulation with step 4 being implemented in the case of mask operator.
We found in practice that the hyperparameter $\lambda$ has little impact on the quality of reconstructions and hence we use the recommended values $\lambda=1$\footnote{\url{https://github.com/yuanzhi-zhu/DiffPIR}}.
On the other hand for the second hyperparameter $\zeta$, we find that using $\zeta=0.3$ yielded the best reconstructions.
}

\paragraph{\rebuttal{\ddnm.}}
\rebuttal{%
We adapt the implementation in the released code\footnote{\url{https://github.com/wyhuai/DDNM}} to the flow matching formulation with the step 4 in \citet[Algorithm 3]{wang2023zeroshot} being implemented for a mask operator.
The official implementation uses a DDIM transition in step 5 of Algorithm 3 whose stochasticity is controlled by the hyperparemters $\eta$.
As recommended, we set the latter to $\eta=0.85$.
}

\paragraph{\flowchef\ \& \flowdps.}
For both algorithms, we adapt the implementations available in the released codes \flowchef\footnote{\url{https://github.com/FlowDPS-Inverse/FlowDPS}} \footnote{\url{https://github.com/FlowChef/flowchef}} to our codebase.
We observe that the two algorithms are quite similar, with \flowdps\ being distinct by adding stochasticity between iterations.

\paragraph{\pnpflow.}
We reimplement \citet[Algorithm 3]{martin2025pnpflow} while taking as a reference the released code\footnote{\url{https://github.com/annegnx/PnP-Flow}}.
For the stepsizes on data fidelity term, we find that a constant scheduler with higher stepsize enables the algorithm to fit the observation, mitigate the smooth and blurring effects in the reconstruction and hence yield better reconstructions.

\paragraph{\psld.}
We implement the \psld\ algorithm provided in \citet[Algorithm 2]{rout2024solving}.
We find that \psld\ algorithm requires several diffusion steps, e.g. at least 150 diffusion steps, to yield good results.
Unfortunately, we were not able to make it work well for the low NFE setup.

\paragraph{\reddiff.}
We implement \citet[Algorithm 1]{mardani2024a} based on the official code\footnote{\url{https://github.com/NVlabs/RED-diff}} and adapt it to the flow matching formulation.
We initialize the algorithm with a sample for a standard Gaussian.
For low NFE setups, we find that using a constant weight schedule yields better results, namely in terms fitting the observation and providing consistent reconstructions.

\paragraph{\resample.}
We reimplemented \citet[Algorithm 1]{song2024solving} based on the provided implementation details in \citet[Appendix]{song2024solving} and the reference code\footnote{\url{https://github.com/soominkwon/resample}}.
As noted in \citet{janati2025mgdm}, we set the tolerance $\varepsilon$ for optimizing the data consistency to the noise level $\std_{\obs}$.
Since we are working with low NEFs, we set the frequency at which hard data consistency is applied (skip step size) to $5$.
That aside, we found that the algorithm requires several diffusion steps (200) in order to output good enough reconstructions. We note that removing the DPS step in the data consistency steps reduces the quality of the reconstructions.

\begin{table}[h!]
\vspace{4mm}
\centering
\captionsetup{font=small}
\caption{Hyperparameters for each algorithm (using the same notations as in their paper) and task variations. “---” indicates identical across tasks.}
\resizebox{\textwidth}{!}{
\begin{tabular}{lclccccc}
\toprule
Algorithm & $n_{\text{steps}}$ & Base hyperparameters & \multicolumn{5}{c}{Latent tasks} \\
\cmidrule(lr){4-8}
& & & Half & Top & Bottom & Center & Strip \\

\midrule
 
\blended & 50 & 
\scriptsize
\begin{tabular}[c]{@{}l@{}} 
\texttt{blending\_percentage} = 0
\end{tabular}
& \scriptsize --- & \scriptsize --- & \scriptsize --- & \scriptsize --- & \scriptsize --- \\

\midrule
\daps & 50 & \scriptsize
\begin{tabular}[c]{@{}l@{}} 
$N_{\text{ode}} = 2$ \\
$\texttt{MCMC steps} = 20$ \\
$\beta_y = 10^{-2}$ \\
$\texttt{Min ratio} = 0.43$ \\
$\texttt{MCMC sampler = Langevin}$ \\
$\rho = 1$
\end{tabular}
& \scriptsize $\eta_0 = 2 \times 10^{-5}$ 
& \scriptsize $\eta_0 = 3 \times 10^{-5}$ 
& \scriptsize $\eta_0 = 2 \times 10^{-5}$ 
& \scriptsize $\eta_0 = 9 \times 10^{-6}$ 
& \scriptsize $\eta_0 = 2 \times 10^{-5}$ \\

\midrule
 \ddnm & 50 & $\eta=0.85$ & \scriptsize --- & \scriptsize --- & \scriptsize --- & \scriptsize --- & \scriptsize --- \\

\midrule
 \diffpir & 50 & \scriptsize
\begin{tabular}[c]{@{}l@{}} 
$\lambda = 1$ \\
$\zeta= 0.3$ \\
\end{tabular}
& \scriptsize --- & \scriptsize --- & \scriptsize --- & \scriptsize --- & \scriptsize --- \\

\midrule
\flowchef & 50 & \scriptsize
\begin{tabular}[c]{@{}l@{}} 
$\texttt{step size} = 0.9$ \\
$\texttt{grad\_descent\_steps} = 10$
\end{tabular}
& \scriptsize --- & \scriptsize --- & \scriptsize --- & \scriptsize --- & \scriptsize --- \\
\midrule
\flowdps & 50 & \scriptsize
\begin{tabular}[c]{@{}l@{}} 
$\texttt{grad\_descent\_steps} = 3$
\end{tabular}
& \scriptsize $\texttt{step\_size} = 20$ 
& \scriptsize $\texttt{step\_size} = 10$ 
& \scriptsize $\texttt{step\_size} = 10$ 
& \scriptsize $\texttt{step\_size} = 10$ 
& \scriptsize $\texttt{step\_size} = 10$ \\
\midrule
\pnpflow & 50 & \scriptsize
\begin{tabular}[c]{@{}l@{}} 
$\alpha = 1.0$ \\
$\texttt{lr style} = \texttt{constant}$
\end{tabular}
& \scriptsize $\gamma_n = 0.8$ 
& \scriptsize $\gamma_n = 1.3$ 
& \scriptsize $\gamma_n = 1.4$ 
& \scriptsize $\gamma_n = 0.8$ 
& \scriptsize $\gamma_n = 0.8$ \\
\midrule
\psld & 50 & \scriptsize
\begin{tabular}[c]{@{}l@{}} 
$\texttt{DDIM\_param} = 1.0$
\end{tabular}
& \scriptsize 
\begin{tabular}[c]{@{}l@{}} 
$\gamma = 0.01$ \\
$\eta = 0.01$
\end{tabular}
& \scriptsize 
\begin{tabular}[c]{@{}l@{}} 
$\gamma = 0.01$ \\
$\eta = 0.01$
\end{tabular}
& \scriptsize 
\begin{tabular}[c]{@{}l@{}} 
$\gamma = 0.01$ \\
$\eta = 0.01$
\end{tabular}
& \scriptsize 
\begin{tabular}[c]{@{}l@{}} 
$\gamma = 0.05$ \\
$\eta = 0.1$
\end{tabular}
& \scriptsize 
\begin{tabular}[c]{@{}l@{}} 
$\gamma = 0.1$ \\
$\eta = 0.5$
\end{tabular} \\
\midrule
\reddiff & 50 & \scriptsize
\begin{tabular}[c]{@{}l@{}} 
$\texttt{lr} = 0.2$ \\
$\texttt{grad\_term\_weight} = 0.25$ \\
$\texttt{obs\_weight} = 1.0$
\end{tabular}
& \scriptsize --- & \scriptsize --- & \scriptsize --- & \scriptsize --- & \scriptsize --- \\

\midrule
\resample & 50 & \scriptsize
\begin{tabular}[c]{@{}l@{}} 
$C = 5$ \\
$\texttt{grad\_descent\_steps} = 200$ \\
$\gamma_{\text{scale}} = 40.0$ \\
$\texttt{lr}_{\text{pixel}} = 10^{-2}$ \\
$\texttt{lr}_{\text{latent}} = 5 \times 10^{-3}$
\end{tabular}
& \scriptsize --- & \scriptsize --- & \scriptsize --- & \scriptsize --- & \scriptsize --- \\

\midrule
\rowcolor{deepcarrotorange!60} \algo\ (ours) & 25 & \scriptsize $\eta = \sigma_s (1 - \alpha_s)$ 
& \scriptsize --- & \scriptsize --- & \scriptsize --- & \scriptsize --- & \scriptsize --- \\

\bottomrule
\end{tabular}
}
\label{tab:hyperparams-tasks}
\end{table}

\section{Examples of reconstructions}
\label{sec:reconstructions}
Here, we provide a side-by-side comparison of the \algo\ and the considered baselines on image editing tasks via inpainting on \PieBench.
The red semi-transparent layer in the first column shows the masked region to be edited and the text in the left-hand side of each row represents the editing prompt.

In the follow work \citet{ghorbel2026ding-editor}, we provide extended experiments that include other datasets, models, as well as the editing task on videos.

\vspace*{5mm}
\emph{(See the next pages for the gallery of examples)}
\vspace*{5mm}

\begin{figure}[h]
    \centering
    \resizebox{1.0\textwidth}{!}{
    \begin{tabular}{@{}c@{\hspace{0.03\textwidth}}c@{}}
        \includegraphics{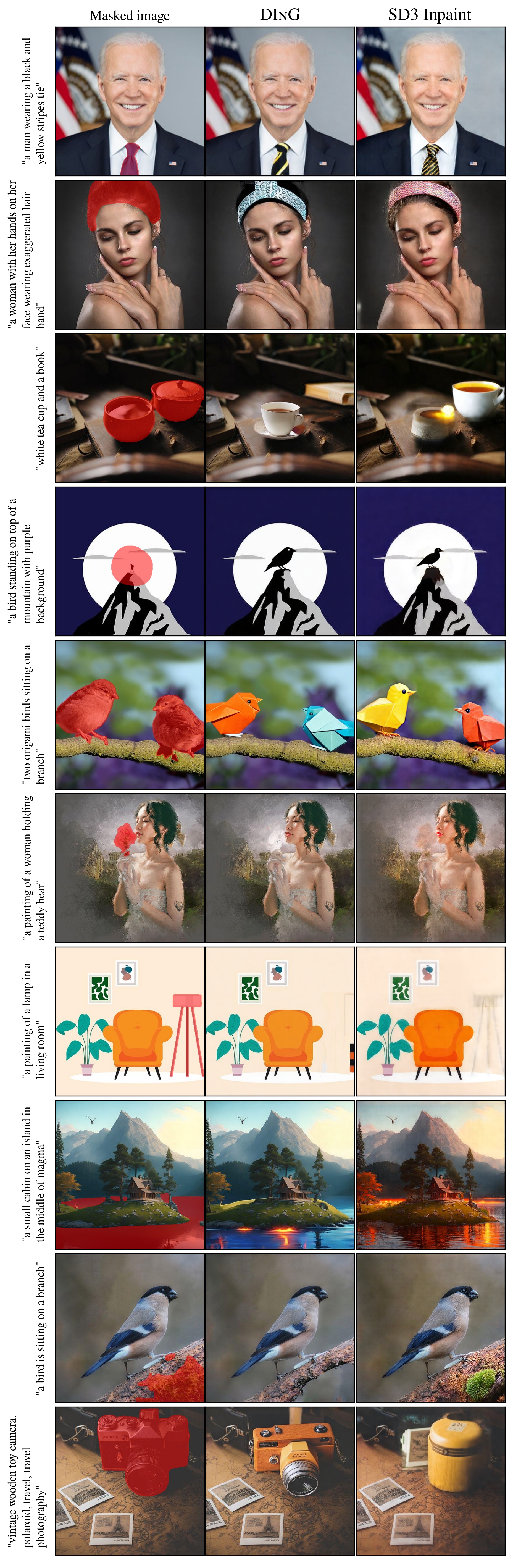} &
        \includegraphics{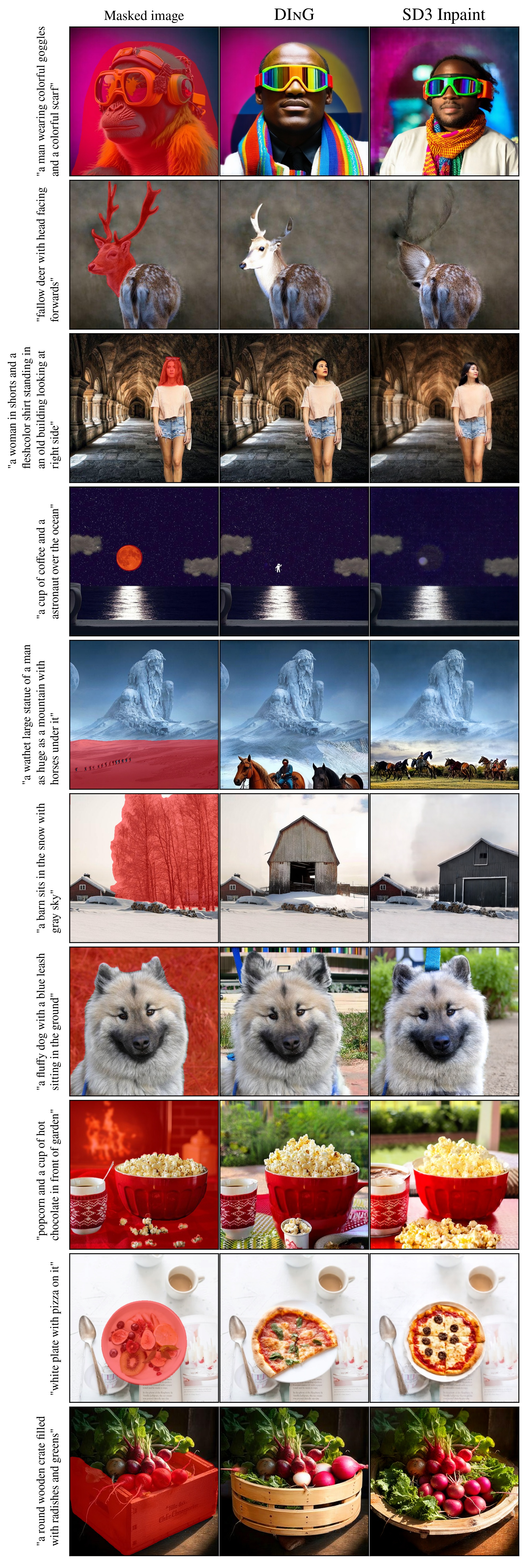} \\
    \end{tabular}
    }
    \caption{Comparison of \algo\ and finetuned SD3 on \PieBench. Both methods have the same runtime of $2.2$s.}
    \label{fig:sd3ft}
\end{figure}

\begin{figure} 
    \centering 
    \includegraphics[width=\textwidth]{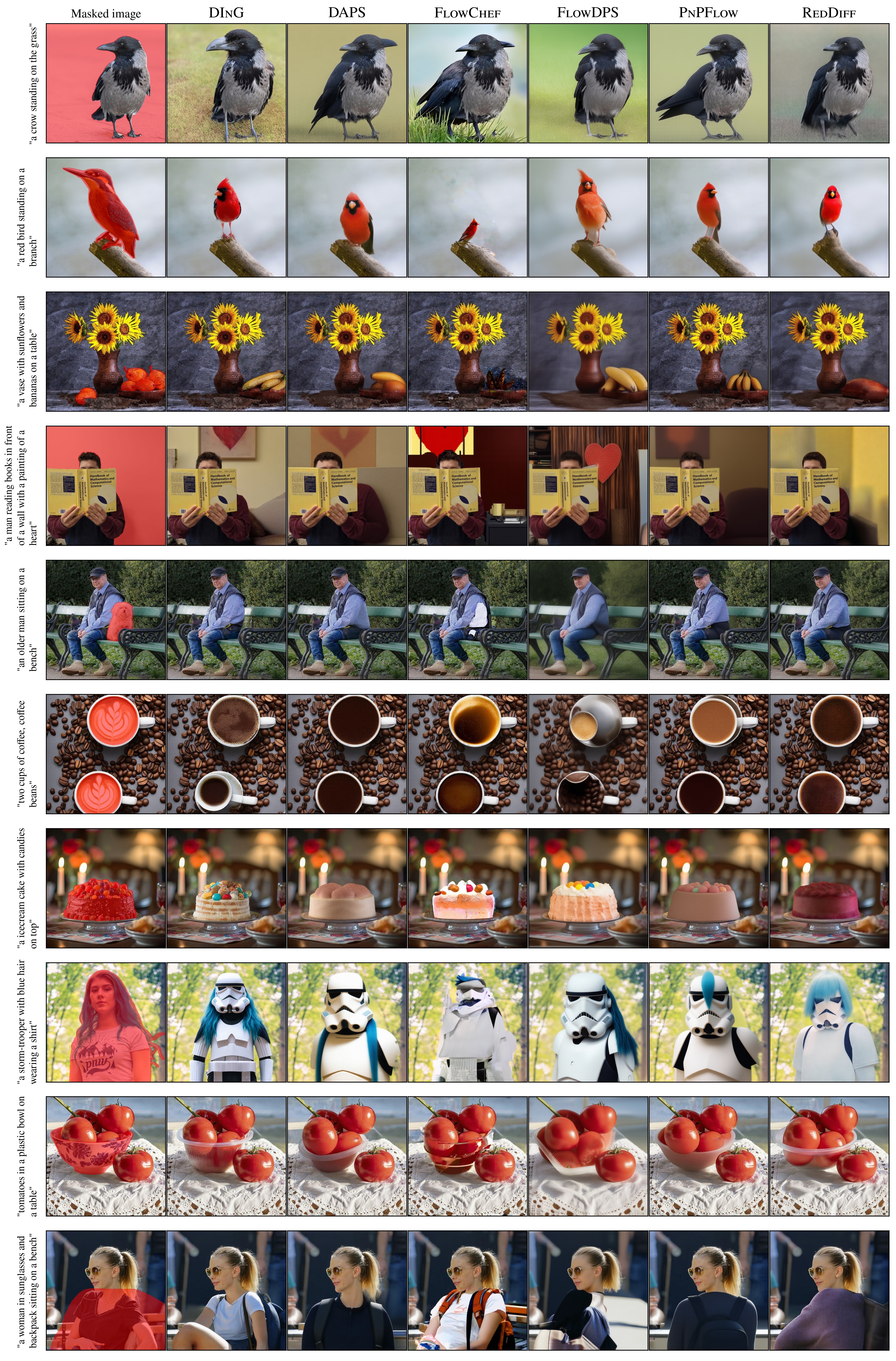}
    \caption{Comparison of \algo\ and zero-shot baselines on \PieBench. All methods use 50 NFEs.}
\end{figure}

\begin{figure} 
    \centering 
    \includegraphics[width=\textwidth]{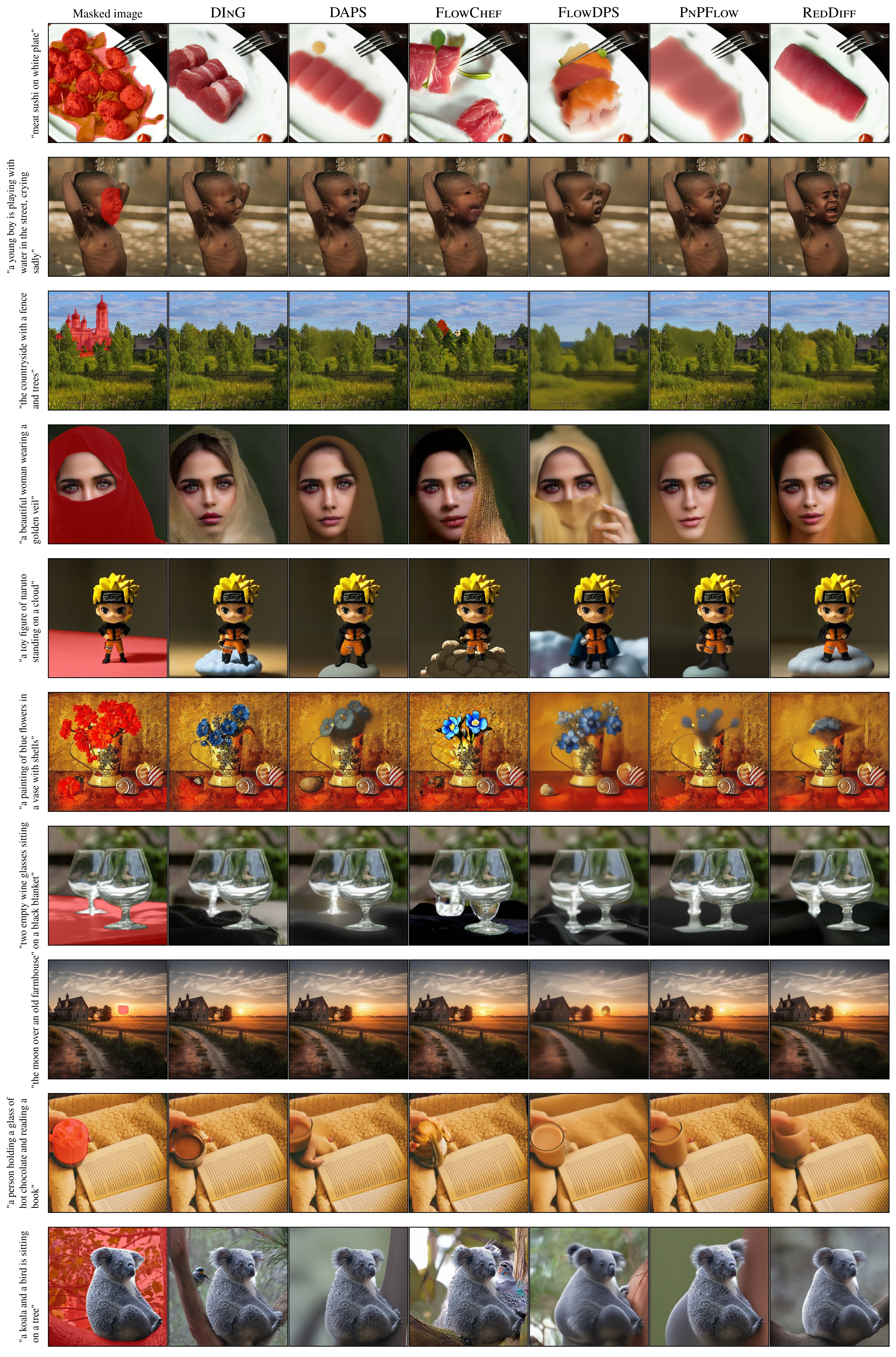}
    \caption{Comparison of \algo\ and zero-shot baselines on \PieBench. All methods use 50 NFEs.}
\end{figure}

\begin{figure} 
    \centering 
    \includegraphics[width=\textwidth]{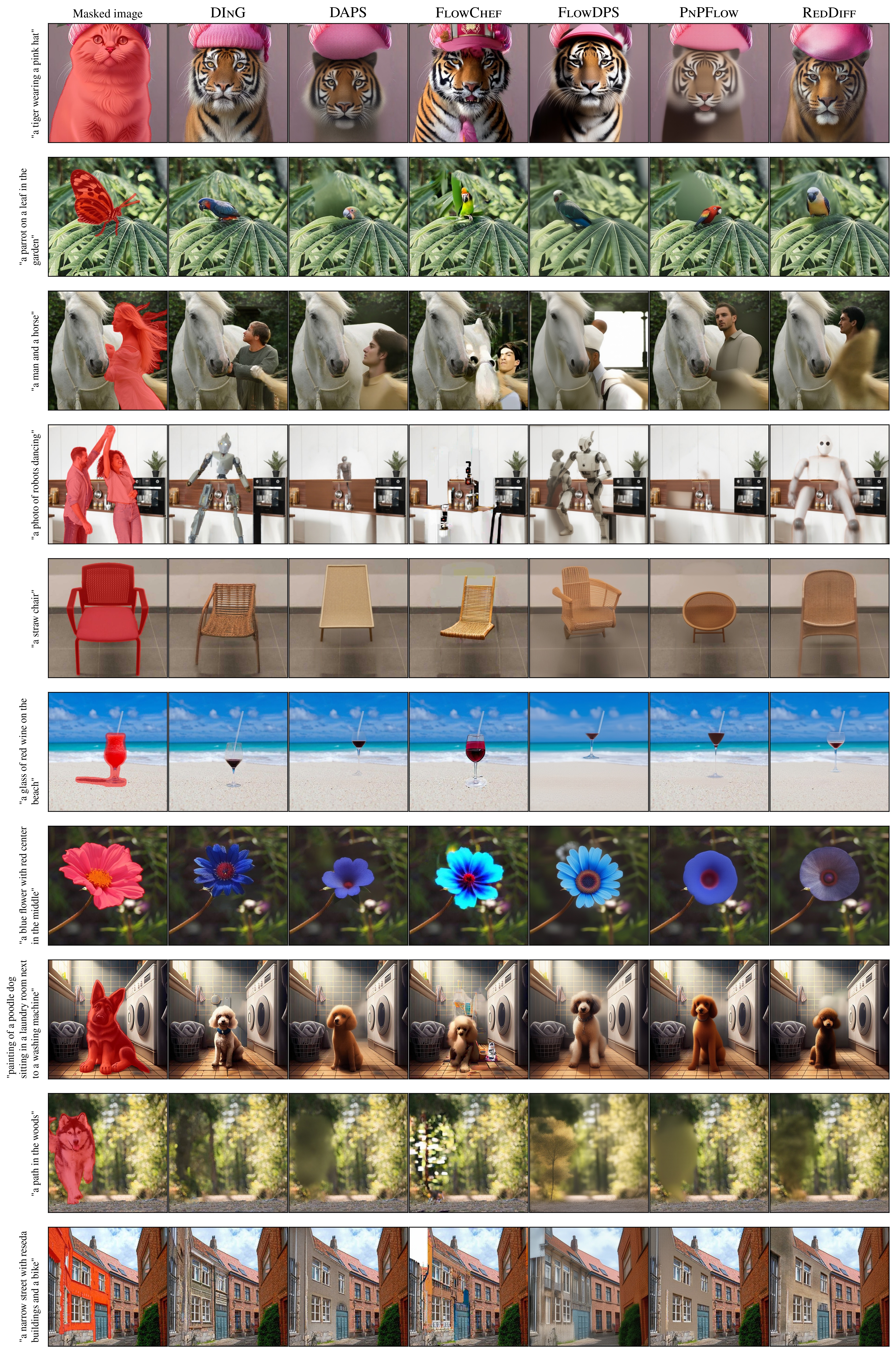}
    \caption{Comparison of \algo\ and zero-shot baselines on \PieBench. All methods use 50 NFEs.}
\end{figure}

\begin{figure} 
    \centering 
    \includegraphics[width=\textwidth]{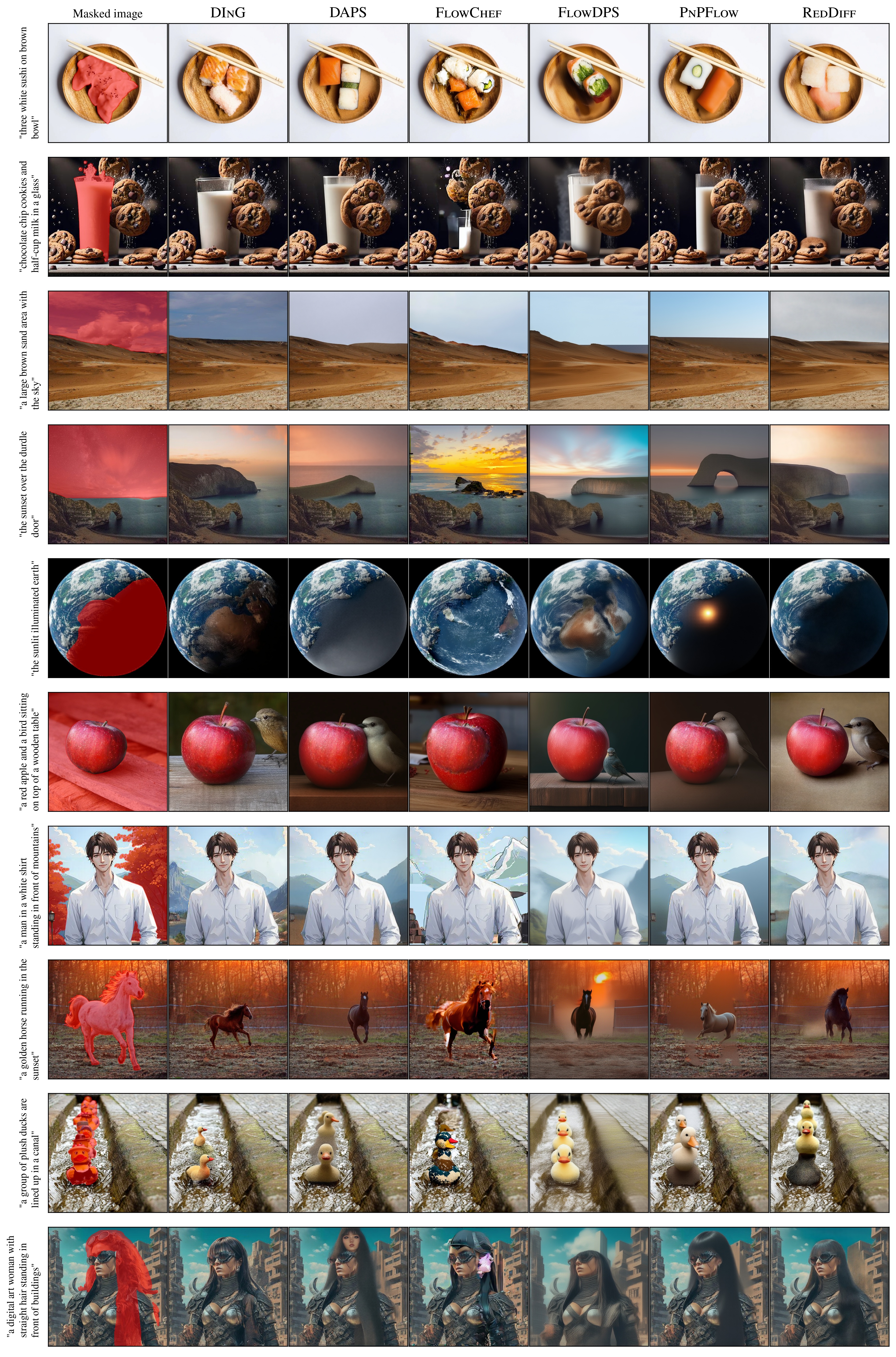}
    \caption{Comparison of \algo\ and zero-shot baselines on \PieBench. All methods use 50 NFEs.}
\end{figure}

\end{document}